\documentclass[conference]{IEEEtran}
\IEEEoverridecommandlockouts
\usepackage{cite}
\usepackage{amsmath,amssymb,amsfonts}
\usepackage{algorithmic}
\usepackage{graphicx}
\usepackage{textcomp}
\usepackage{xcolor}
\usepackage{bm}
\usepackage{booktabs}
\usepackage{pifont}
\usepackage{subfigure}
\usepackage{amsthm}
\usepackage{multirow}

\newtheorem{theorem}{Theorem}

\newtheorem{corollary}{Corollary}

\usepackage{tikz}
\newcommand*{\circled}[1]{\lower.7ex\hbox{\tikz\draw (0pt, 0pt)%
		circle (.5em) node {\makebox[1em][c]{\small #1}};}}

\def\BibTeX{{\rm B\kern-.05em{\sc i\kern-.025em b}\kern-.08em
    T\kern-.1667em\lower.7ex\hbox{E}\kern-.125emX}}
\begin{document}

\title{
	LineaRE: Simple but Powerful Knowledge Graph Embedding for Link Prediction
	\thanks{The corresponding author of the paper is Jing Zhang. The source code is available at: https://github.com/pengyanhui/LineaRE.}
}

\author{
	\IEEEauthorblockN{
		Yanhui Peng, Jing Zhang\textsuperscript{*}
	}
	\IEEEauthorblockA{
		\textit{School of Computer Science and Engineering} \\
		\textit{Nanjing University of Science and Technology} \\
		200 Xiaolingwei Street, Nanjing 210094, China \\
		\{118106010712, jzhang\}@njust.edu.cn
	}
}

\maketitle

\begin{abstract}
The task of link prediction for knowledge graphs is to predict missing relationships between entities. Knowledge graph embedding, which aims to represent entities and relations of a knowledge graph as low dimensional vectors in a continuous vector space, has achieved promising predictive performance. If an embedding model can cover different types of connectivity patterns and mapping properties of relations as many as possible, it will potentially bring more benefits for link prediction tasks. In this paper, we propose a novel embedding model, namely LineaRE, which is capable of modeling four connectivity patterns (i.e., symmetry, antisymmetry, inversion, and composition) and four mapping properties (i.e., one-to-one, one-to-many, many-to-one, and many-to-many) of relations. Specifically, we regard knowledge graph embedding as a simple linear regression task, where a relation is modeled as a linear function of two low-dimensional vector-presented entities with two weight vectors and a bias vector. Since the vectors are defined in a real number space and the scoring function of the model is linear, our model is simple and scalable to large knowledge graphs. Experimental results on multiple widely used real-world datasets show that the proposed LineaRE model significantly outperforms existing state-of-the-art models for link prediction tasks.
\end{abstract}


\begin{table*}[t]
	\caption{
		The score functions $f_r(\bm{h},\bm{t})$ of several knowledge graph embedding models.
	}
	\label{ScoringFunction}
	\centering
	\begin{tabular}{|c|c|c|}  
		\hline
		\textbf{Model}
		& \textbf{Scoring function $f_r(\bm{h}, \bm{h})$}
		& \textbf{\# Parameters} \\
		\hline
		TransE \cite{TransE}
		& $\left \| \bm{h} + \bm{r} - \bm{t} \right \|_{1/2}$
		& $\bm{h}, \bm{t}, \bm{r} \in \mathbb{R}^{k}$ \\
		\hline
		TransH \cite{TransH}
		& $\left \| (\bm{h} - \bm{w_{r}^{\top}hw_{r}}) + \bm{d_r}-(\bm{t} - \bm{w_{r}^{\top}tw_{r}}) \right \|_{2}^{2}$
		& $\bm{h}, \bm{t}, \bm{w_r}, \bm{d_r} \in \mathbb{R}^{k}$ \\
		\hline
		TransR \cite{TransR}
		& $\left \| \bm{M_rh} + \bm{r} - \bm{M_rt} \right \|_{2}^{2}$
		& $\bm{h}, \bm{t} \in \mathbb{R}^{k}, \bm{r} \in \mathbb{R}^{d}, \bm{M_r} \in \mathbb{R}^{k \times d}$ \\
		\hline
		TransD \cite{TransD}
		& $\left \| (\bm{r_p h_p^{\top}} + \bm{I^{k \times d}})\bm{h} + \bm{r} - (\bm{r_p t_p^{\top}} + \bm{I^{k \times d}})\bm{t} \right \|_{2}^{2}$
		& $\bm{h}, \bm{h_p}, \bm{t}, \bm{t_p} \in \mathbb{R}^{d}, \bm{r}, \bm{r_p} \in \mathbb{R}^{k}$ \\
		\hline
		DistMult \cite{DistMult}
		& $\bm{h^{\top}}diag(\bm{r})\bm{t}$
		& $\bm{h}, \bm{t}, \bm{r} \in \mathbb{R}^{k}$ \\
		\hline
		ComplEx \cite{ComplEx}
		& $Re(\bm{h^{\top}}diag(\bm{r})\bm{\bar{t}})$
		& $\bm{h}, \bm{t}, \bm{r} \in \mathbb{C}^{k}$ \\
		\hline
		ConvE \cite{ConvE}
		& $<\sigma(vec(\sigma([\bar{\bm{r}}, \bar{\bm{h}}] * \varOmega)) \bm{W}), \bm{t}>$
		& $\bm{h}, \bm{t}, \bm{r} \in \mathbb{R}^{k}$ \\
		\hline
		RotatE \cite{RotatE}
		& $\left \| \bm{h} \circ \bm{r} - \bm{t} \right \|_{1}$
		& $\bm{h}, \bm{t}, \bm{r} \in \mathbb{C}^{k}, |r_i|=1$ \\
		\hline
		LineaRE (Our model)
		& $\left \| \bm{w_{r}^{1}} \circ \bm{h} + \bm{b_r} - \bm{w_{r}^{2}} \circ \bm{t} \right \|_{1}$
		& $\bm{h}, \bm{t}, \bm{b_r}, \bm{w_{r}^{1}}, \bm{w_{r}^{2}}, \in \mathbb{R}^{k}$ \\
		\hline
		\multicolumn{3}{l}{$<\cdot>$ denotes the generalized dot product.} \\
		\multicolumn{3}{l}{$\circ$ denotes the Hadamard product.} \\
		\multicolumn{3}{l}{$\sigma$ denotes activation function and $*$ denotes 2D convolution.} \\
		\multicolumn{3}{l}{$\bar{\cdot}$ denotes conjugate for complex vectors, and 2D reshaping for real vectors in ConvE model.}
	\end{tabular}
\end{table*}
\begin{table*}[t]
	\caption{
		The modeling capabilities of models.
	}
	\label{ModelingAbility}
	\centering
	\begin{tabular}{|c|c|c|c|c|c|}
		\hline
		\textbf{Model}
		& \textbf{Symmetry}
		& \textbf{Antisymmetry}
		& \textbf{Inversion}
		& \textbf{Composition}
		& \textbf{Complex mapping properties} \\
		\hline
		TransE	& \textendash	& \ding{51}	& \ding{51}	& \ding{51}	& \textendash \\
		\hline
		TransX$^{\mathrm{a}}$	& \ding{51}	& \ding{51}	& \textendash	& \textendash	& \ding{51} \\
		\hline
		DistMult& \ding{51}	& \textendash	& \textendash	& \textendash	& \ding{51} \\
		\hline
		ComplEx	& \ding{51}	& \ding{51}	& \ding{51}	& \textendash	& \ding{51} \\
		\hline
		RotatE 	& \ding{51}	& \ding{51}	& \ding{51}	& \ding{51}	& \textendash \\
		\hline
		LineaRE (Our model)	& \ding{51}	& \ding{51}	& \ding{51}	& \ding{51}	& \ding{51} \\
		\hline
		\multicolumn{6}{l}{
			$^{\mathrm{a}}$ TransX represents a wide range of TransE’s \cite{TransE} variants, such as TransH \cite{TransH}, TransR \cite{TransR}, and TransD \cite{TransD}.
		}  \\
		\multicolumn{6}{l}{
			Results with respect to connectivity patterns are taken from \cite{RotatE}.
		}
	\end{tabular}
\end{table*}
\section{Introduction} \label{Introduction}
The construction and applications of knowledge graphs have attracted much attention in recent years. Many knowledge graphs, such as WordNet \cite{WordNet}, DBpedia \cite{DBpedia}, and Freebase \cite{Freebase}, have been built and successfully applied to some AI domains, including information retrieval \cite{Retrieval}, recommender systems \cite{Recommender}, question-answering systems \cite{HaoQA,HuangQA}, and natural language processing \cite{NLP}. A large knowledge graph stores billions of factual triplets in the form of directed graphs, where each triplet in the form of (\textit{head\_entity, relation, tail\_entity}) (denoted by ($h$, $r$, $t$) in this paper) stands for an edge with two end nodes in the graph, indicating that there exists a specific relationship between the head and tail entities. However, knowledge graphs still suffer from incompleteness, and link prediction, which predicts relations between entities according to existing triplets, is an important way to knowledge completion \cite{bootstrapping,Trustworthiness}. On a graph with this kind of symbolic representation, algorithms that compute semantic relationships between entities usually have high computational complexity and lack scalability. Therefore, knowledge graph embedding is proposed to improve the calculation efficiency. By embedding entities and relations into a low-dimensional vector space, we can efficiently implement the operations such as the calculation of semantic similarity between entities, which is of considerable significance to the completion, reasoning, and applications of knowledge graphs.

Quite a few methods \cite{TransE,ComplEx,ConvE,RotatE,paths,simple} have been proposed for knowledge graph embedding. Given a knowledge graph, these methods first assign one or more vectors (or matrices) to each entity and relation, then define a scoring function $f_{r}(\bm{h},\bm{t})$ to measure the plausibility of each triplet, and finally maximize the global plausibility of all triplets. Thus, scoring functions play a critical role in the methods, which determine the capability and computational complexity of models. The capability of a model is primarily influenced by the variety of \textit{connectivity patterns} and \textit{mapping properties} of relations it can model. In a knowledge graph, following \cite{RotatE}, we have four connectivity patterns of relations:
\begin{itemize}
	\item \textbf{Symmetry}. A relation $r$ is symmetric if
	\begin{center}
		$\forall x, y: r(x,y) \Rightarrow r(y,x)$
	\end{center}
	\item \textbf{Antisymmetry}. A relation $r$ is antisymmetric if 
	\begin{center}
		$\forall x, y: r(x,y) \Rightarrow \neg r(y,x)$
	\end{center}
	\item \textbf{Inversion}. Relation $r_{1}$ is inverse to relation $r_{2}$ if
	\begin{center}
		$\forall x, y: r_{1}(x,y) \Rightarrow r_{2}(y,x)$
	\end{center}
	\item \textbf{Composition}. Relation $r_{1}$ is composed of relation $r_{2}$ and relation $r_{3}$ if
	\begin{center}
		$\forall x, y, z: r_{2}(x,y) \wedge r_{3}(y,z) \Rightarrow r_{1}(x,z)$
	\end{center}
\end{itemize}
Also, following \cite{TransE}, we have four mapping properties of relations:
\begin{itemize}
	\item \textbf{One-to-One (1-to-1)}. Relation $r$ is 1-to-1 if a head can appear with at most one tail.
	\item \textbf{One-to-Many (1-to-N)}. Relation $r$ is 1-to-N if a head can appear with many tails.
	\item \textbf{Many-to-One (N-to-1)}. Relation $r$ is N-to-1 if many heads can appear with the same tail.
	\item \textbf{Many-to-Many (N-to-N)}. Relation $r$ is N-to-N if many heads can appear with many tails.
\end{itemize}
We call the latter three relations (i.e., 1-to-N, N-to-1, and N-to-N ) as the \textit{complex mapping properties}.

If an embedding method could model connectivity patterns and mapping properties as many as possible, it would potentially benefit the link prediction task. This is because methods with stronger modeling ability can preserve more structural information of knowledge graphs, so that the embeddings of entities and relations have more precise semantics. For example, in a link prediction task, a model has learned that the relation \textit{Nationality} is a \textit{Composition} of \textit{BornIn} and \textit{LocatedIn}. When triplets (Tom, \textit{BornIn}, New York), (New York, \textit{LocatedIn}, United States) both hold, it can infer that triplet (Tom, \textit{Nationality}, United States) holds. Another negative instance is that if a method cannot model N-to-1 mapping property, it probably treats Leonardo DiCaprio and Kate Winslet as the same entity when it reads relations (Leonardo DiCaprio, \textit{ActorIn}, Titanic) and (Kate Winslet, \textit{ActorIn}, Titanic).

In this paper, we proposed a novel method, namely Linear Regression Embedding (LineaRE), which interprets a relation as a linear function of entities head and tail. Specifically, our model represents each entity as a low-dimensional vector (denoted by $\bm{h}$ or $\bm{t}$), and each relation as two weight vectors and a bias vector (denoted by $\bm{w_{r}^{1}}$, $\bm{w_{r}^{2}}$, and $\bm{b_r}$), where $\bm{h}$, $\bm{t}$, $\bm{w_{r}^{1}}$, $\bm{w_{r}^{2}}$, and $\bm{b_r} \in \mathbb{R}^{k}$. Given a golden triplet ($h$, $r$, $t$), we expect the equation
$
	\bm{w_{r}^{1}} \circ \bm{h}
	+
	\bm{b_r}
	=
	\bm{w_{r}^{2}} \circ \bm{t}
$,
where $\circ$ denotes the Hadamard (element-wise) product
\footnote{
	Given two vectors $\bm{x}$ and $\bm{y}$, $[\bm{x} \circ \bm{y}]_i = [\bm{x}]_i \cdot [\bm{y}]_i$.
}. Tables \ref{ScoringFunction} \& \ref{ModelingAbility} summarize the scoring functions and the modeling capabilities of some state-of-the-art knowledge graph embedding methods, respectively. Table \ref{ScoringFunction} shows that, the parameters of ComplEx and RotatE are defined in complex number spaces and those of the others (including our model) are defined in real number spaces. Compared with most of the other models, the scoring function of our LineaRE is simpler. Table \ref{ModelingAbility} shows that, some of them (such as TransE and RotatE) are better at modeling connectivity patterns but do not consider complex mapping properties. In contrast, some others (TransH and DistMult) are better at modeling complex mapping properties but sacrifice some capability to model connectivity patterns. Our LineaRE has the most comprehensive modeling capability.

We summarize the main contributions of this paper as follows:
\begin{enumerate}
	\item We propose a novel method LineaRE for knowledge graph embedding, which is simple and can cover all the above connectivity patterns and mapping properties.
	\item We provide formal mathematical proofs to demonstrate the modeling capabilities of LineaRE.
	\item We conduct extensive experiments to evaluate our LineaRE on the task of link prediction on several benchmark datasets. The experimental results show that LineaRE has significant improvements compared with the existing state-of-the-art methods.
\end{enumerate}

\section{Related Work}
Knowledge graph embedding models can be roughly categorized into two groups \cite{Survey}: \textit{translational distance models} and \textit{semantic matching models}. In this section, we will briefly describe some related models and the differences between our models and them.

\subsection{Translational Distance Models}
The most classic translational distance model is TransE \cite{TransE}. Given a triplet ($h$, $r$, $t$), TransE interprets the relation as a translation $r$ from the head entity $h$ to the tail entity $t$, i.e., $\bm{h} + \bm{r} \approx \bm{t}$, where $\bm{h}, \bm{t}, \bm{r} \in \mathbb{R}^{k}$. When a relation is symmetric, its vector will be represented by $\bm{0}$, resulting in TransE being unable to distinguish different symmetric relations. In addition, TransE has issues in dealing with 1-to-N, N-to-1, and N-to-N relations. TransH \cite{TransH} was proposed to address the issues of TransE in modeling complex mapping properties, which interprets a relation as a translating operation $\bm{d_r}$ on a hyperplane (defined by a relation-specific normal vector $\bm{w_r}$). For a triplet ($h$, $r$, $t$), the embeddings $\bm{h}$ and $\bm{t}$ are first projected to the hyperplane, and then connected by $\bm{d_r}$. TransR \cite{TransR} supposes that entities have multiple aspects and various relations may focus on different aspects of entities. So, entity embeddings are first projected from entity space to corresponding relation space by a relation-specific projecting matrix $\bm{M_r}$, and then connected by a translation vector $\bm{r}$. Because of the high space and time complexity of matrix operation, TransR cannot be applied to large-scale knowledge graphs. TransD \cite{TransD} is an improvement of TransR, which uses entity projecting vectors ($\bm{h_p}$ and $\bm{t_p}$) and relation projecting vectors ($\bm{r_p}$) to construct mapping matrices dynamically. However, reference \cite{RotatE} has proved that such projection-based variants of TransE can not model inversion and composition patterns.

RotatE \cite{RotatE} represents each entity (relation) as a complex vector, and interprets the relation as a rotation from the head entity to the tail entity on the complex plane for a triplet. RotatE can model all the above connectivity patterns, but does not consider the complex mapping properties.

\subsection{Semantic Matching Models}
Semantic matching models measure plausibility of facts by matching latent semantics of entities and relations embodied in their vector space representations \cite{Survey}. These models can be further divided into two categories: \textit{bilinear models} and \textit{neural network-based models}.

\paragraph{Bilinear Models}
RESCAL \cite{RESCAL} is a classic bilinear model, which represents each entity as a vector and each relation as a full rank matrix which models pairwise interactions between latent factors, and the score function is defined as $f_r(\bm{h},\bm{t}) = \bm{h^{\top}M_rt}$. Obviously, such a method has two serious defects: prone to overfitting and can not be applied to large-scale knowledge graphs.
DistMult \cite{DistMult} simplifies RESCAL. For a triplet ($h$, $r$, $t$), the relation is represented as a diagonal matrix to capture pairwise interactions between the components of $\bm{h}$ and $\bm{t}$ along the same dimension, i.e., $\bm{M_r}$ is restricted to a diagonal matrix. However, $f_r(\bm{h},\bm{t}) = \bm{h^{\top}diag(r)t} = \bm{t^{\top}diag(r)h} = f_r(\bm{t},\bm{h})$ for any $\bm{h}$ and $\bm{t}$. As a result, DistMult can only deal with symmetric relations. ComplEx \cite{ComplEx} was proposed to address the issues of DistMult in modeling antisymmetric relations by introducing complex-valued embeddings. Unfortunately, ComplEx is still not capable of modeling the composition pattern, and the space and time complexity of the model are considerably increased. 

\paragraph{Neural Network Models}
Some semantic matching methods using neural network architectures have also made good progress in recent years. SME \cite{SME} employs two linear networks to capture the semantics of ($h$, $r$) and ($t$, $r$), respectively, and then measures the plausibility of the whole triplet by the inner product of the two network outputs. While non-linear fully connected neural networks are used in MLP \cite{MLP} and NTN \cite{NTN}. ConvE \cite{ConvE} is a multi-layer convolutional network model. The convolution operation is capable of extracting the feature interactions between the two embeddings $\bm{h}$ and $\bm{t}$. 

\subsection{Difference between LineaRE and Others}
From the perspective of the form of the scoring function, our proposed model LineaRE is formally similar to TransE and belongs to translational distance models. However, LineaRE is \textit{essentially different} from other variants of TransE such as TransH \cite{TransH}, TransR \cite{TransR}, and TransD \cite{TransD} which project entities onto a plane or into a specific vector space.

We regard knowledge graph embedding as a linear regression task. Given an observed triplet ($h$, $r$, $t$), the entity pair ($h$, $t$) is treated as a point, and the relation $r$ is treated as a linear mapping between $h$ and $t$. In fact, we can even regard TransE as a linear regression task, in which the slope is fixed to $\bm{1}$, but TransH \cite{TransH}, TransR \cite{TransR}, and TransD \cite{TransD} can not. One may argue that LineaRE is very similar to TransR if the matrix $\bm{M_r}$ is constrained as a diagonal matrix. However, the slope will be also fixed to $\bm{1}$ in such a simplified TransR.

\section{The Proposed Method}
In this section, we will describe our proposed model LineaRE in detail. First, we provide the formal definition of LineaRE and mathematically prove the powerful modeling capabilities of LineaRE with respect to the connectivity patterns and mapping properties. Then, we introduce the loss function used in our method.

\subsection{Linear Regression Embedding}
We treat knowledge graph embedding as a \textit{Linear Regression} task. A knowledge graph is a directed graph $G=(E,R,T)$, where $E$ is the set of entities, $R$ is the set of relations, and $T=\{ (h, r, t) \}$ denotes the set of all observed triplets. Let $HT_r=\{ (h, t) | (h, r, t) \in T\}$ denote all the entity pairs involving a specific relation $r$. Our main idea is that entity pairs $HT_r$ are treated as points, and the relation $r$ is treated as a linear mapping between head entities and tail entities.
For example, \{(\textit{football}, \textit{\textit{\_has\_part}}, \textit{foot}), (\textit{classroom}, \textit{\_has\_part}, \textit{room})\} is the set of triplets involving relation \textit{\_has\_part}. Our objective is to make the points (\textit{football}, \textit{foot}) and (\textit{classroom}, \textit{room}) lie on the straight line defined by \textit{\_has\_part} in the rectangular coordinate system.

Specifically, we represent each entity as a low-dimensional vector ($\bm{h}$ or $\bm{t}$), and each relation as two weight vectors ($\bm{w_r^{1}}$, $\bm{w_r^{2}}$) and a bias vector ($\bm{b_r}$), where $\bm{h}$, $\bm{t}$, $\bm{w_r^{1}}$, $\bm{w_r^{2}}$, and $\bm{b_r} \in \mathbb{R}^{k}$. Given a golden triplet ($h$, $r$, $t$), we expect
\begin{equation}
	\bm{w_{r}^{1}} \circ \bm{h}
	+
	\bm{b_r}
	=
	\bm{w_{r}^{2}} \circ \bm{t}
\end{equation}
where $\circ$ denotes the Hadamard (element-wise) product. Fig. \ref{Illustration} provides a simple illustration of LineaRE.
\begin{figure}[t]
	\centering
	\includegraphics[width=0.49\textwidth]{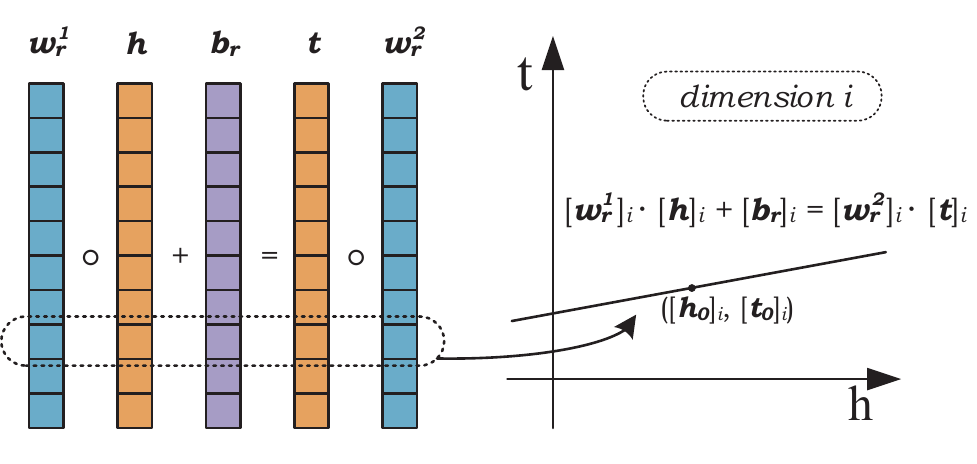}
	\caption{
		Simple illustration of LineaRE. The left part is vectorized representation of entities and relations, and the right part illustrates how one dimension of the vectors is represented in a rectangular coordinate system. The point ($[\bm{h_0}]_i$, $[\bm{t_0}]_i$) lies on the straight line defined by ($[\bm{b_{r}}]_i$, $[\bm{w_{r}^{1}}]_i$, $[\bm{w_{r}^{2}}]_i$) if the triplet ($h_0$, $r$, $t_0$) holds.
	}
	\label{Illustration}
\end{figure}
For dimension $i$, we have
$
\bm{[w_{r}^{1}]_i} \cdot \bm{[h]_i}
+
\bm{[b_r]_i}
=
\bm{[w_{r}^{2}]_i} \cdot \bm{[t]_i}
$, i.e., each dimension of the relation can be represented as a straight line in a rectangular coordinate system, and the point ($\bm{[h]_i}$, $\bm{[t]_i}$) should lie on the straight line.

The scoring function of LineaRE is:
\begin{equation}
	f_{r}(\bm{h},\bm{t})
	=
	\left \|
	\bm{w_r^{1}} \circ \bm{h}
	+
	\bm{b_r}
	-
	\bm{w_r^{2}} \circ \bm{t}
	\right \|_1
\end{equation}
where $\|\bm{x}\|_1$ denotes the L1-Norm of vector $\bm{x}$. We expect a lower score for observed triplets and a higher score for negative triplets which do not hold.

\subsection{Modeling Capabilities of LineaRE} \label{Capabilities}
The connectivity patterns and mapping properties of relations are implicit in the properties of the straight lines. Formally, we have main results as follows:
\begin{theorem}
	LineaRE can model symmetry, antisymmetry, inversion and composition patterns.
\end{theorem}
\begin{figure*}[t]
	\centering
	\subfigure[Symmetry.]{
		\includegraphics[width=0.23\linewidth]{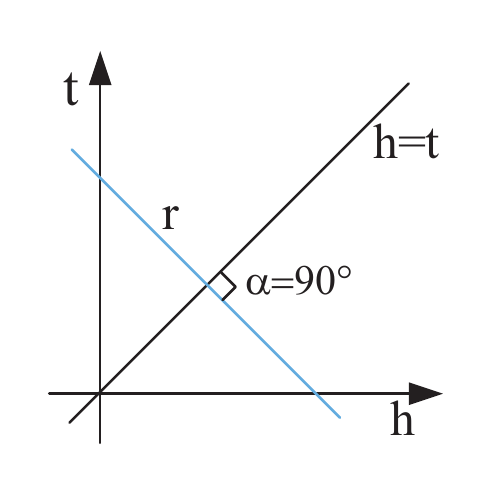}
		\label{Symmetry}
	}
	\subfigure[Antisymmetry.]{
		\includegraphics[width=0.23\linewidth]{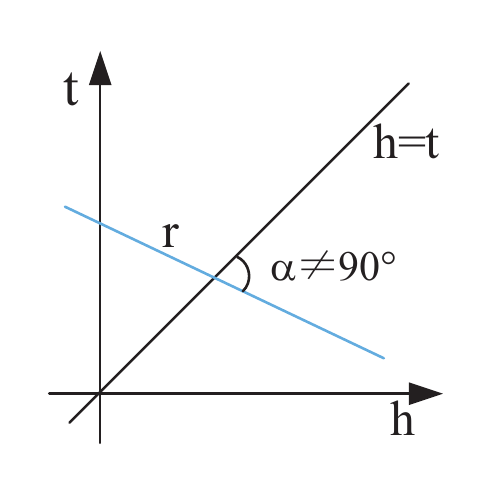}
		\label{Antisymmetry}
	}
	\subfigure[Inversion.]{
		\includegraphics[width=0.23\linewidth]{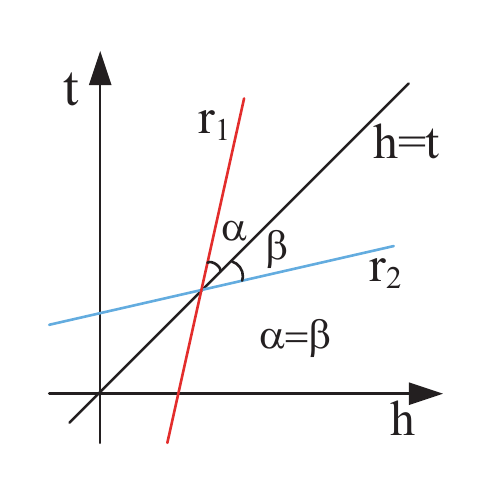}
		\label{Inversion}
	}
	\subfigure[Composition.]{
		\includegraphics[width=0.23\linewidth]{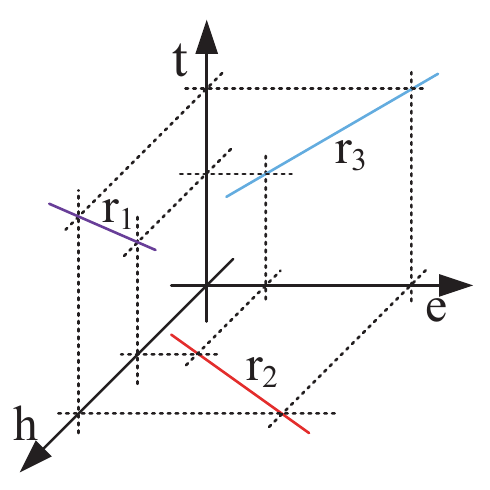}
		\label{Composition}
	}
	\caption{
		Illustrations of LineaRE modeling connectivity patterns.
	}
	\label{Pattern}
\end{figure*}
\begin{figure}[t]
	\centering
	\subfigure[1-to-N.]{
		\includegraphics[width=0.465\linewidth]{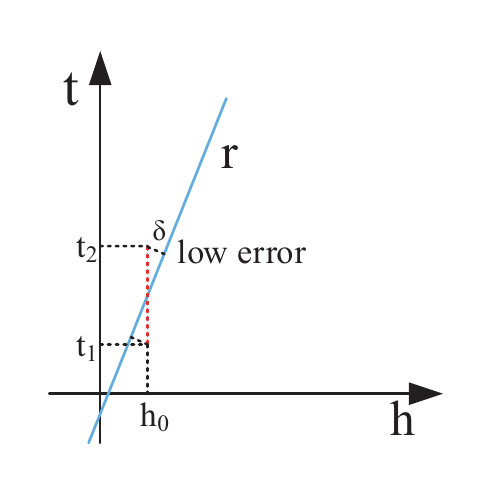}
		\label{1-to-N}
	}
	\subfigure[Extreme 1-to-N.]{
		\includegraphics[width=0.465\linewidth]{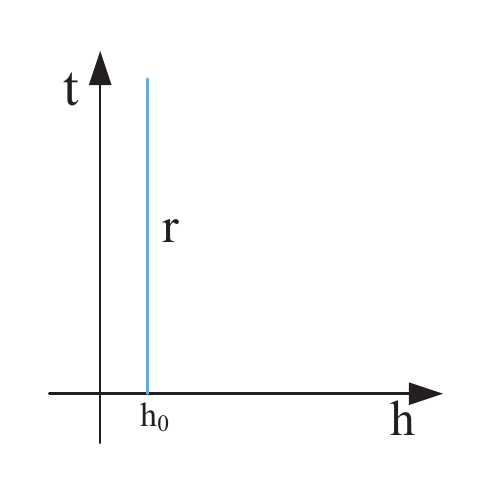}
		\label{1-to-N-any}
	}
	\caption{
		Illustrations of LineaRE modeling complex mapping properties.
	}
	\label{Mapping}
\end{figure}
\begin{proof}
	With $\bm{h}$ and $\bm{t}$ as the axes, LineaRE represents each dimension of a relation as a straight line in the rectangular coordinate system. Fig. \ref{Pattern} illustrates the LineaRE model in a one-dimensional case.
	\begin{itemize}
		\item \textbf{Symmetry} (Each straight line of the relation is symmetrical with respect to $\bm{h} = \bm{t}$, shown in Fig. \ref{Symmetry}).
		\begin{flushleft}
			$
			\left\{
			\begin{aligned}
			&\bm{w_r^{1}} \circ \bm{h} + \bm{b_r} = \bm{w_r^{2}} \circ \bm{t}  \\
			&\bm{w_r^{1}} \circ \bm{t} + \bm{b_r} = \bm{w_r^{2}} \circ \bm{h}  \\
			\end{aligned}
			\right.
			\Leftrightarrow
			\left\{
			\begin{aligned}
			&\bm{w_r^{1}} = -\bm{w_r^{2}} &\circled{1}  \\
			&or  \\
			&\bm{h} = \bm{t} &\circled{2} \\
			\end{aligned}
			\right.
			$
		\end{flushleft}
		$\circled{1}$ When
		$\bm{w_r^{1}} = -\bm{w_r^{2}}$ holds, $\bm{w_r^{1}} \circ (\bm{h} + \bm{t}) + \bm{b_r} = \bm{0}$, then we have $(\bm{h} + \bm{t}) = \bm{\xi}$, where $\bm{\xi}$ is a constant vector. The slope of the straight line is $-1$.
		
		$\circled{2}$ When
		$\bm{h} = \bm{t}$ holds, $ (\bm{w_r^{1}} - \bm{w_r^{2}}) \circ \bm{h} + \bm{b_r} = \bm{0}$, then we have
		$\bm{w_r^{1}} = \bm{w_r^{2}} \wedge \bm{b_r} = \bm{0}$, i.e., the slope is $1$ and the intercept is $0$.
		
		To sum up, when
		$\bm{w_r^{1}} = -\bm{w_r^{2}}$
		or
		$\bm{w_r^{1}} = \bm{w_r^{2}} \wedge \bm{b_r} = \bm{0}$,
		LineaRE can model symmetry pattern. \\
		
		\item \textbf{Antisymmetry} (There exist some straight lines not symmetrical with respect to $\bm{h} = \bm{t}$ in the relation, shown in Fig. \ref{Antisymmetry}).
		
		When
		$\bm{w_r^{1}} \neq -\bm{w_r^{2}}$
		and
		$\bm{w_r^{1}} \neq \bm{w_r^{2}} \vee \bm{b_r} \neq \bm{0}$,
		LineaRE can model symmetry pattern. \\
		In other words, $\exists i \in [1, k]$, $[\bm{w_r^{1}}]_i \neq - [\bm{w_r^{2}}]_i$
		and
		$[\bm{w_r^{1}}]_i \neq [\bm{w_r^{2}}]_i \vee [\bm{b_r}]_i \neq 0$.
		
		\item \textbf{Inversion} (The straight lines of $r_1$ and $r_2$ along the same dimension are symmetrical with respect to $\bm{h} = \bm{t}$, shown in Fig. \ref{Inversion}).
		\begin{flushleft}
			$
			\left\{
			\begin{aligned}
			&\bm{w_{r_1}^{1}} \circ \bm{h} + \bm{b_{r_1}} = \bm{w_{r_1}^{2}} \circ \bm{t}  \\
			&\bm{w_{r_2}^{1}} \circ \bm{t} + \bm{b_{r_2}} = \bm{w_{r_2}^{2}} \circ \bm{h}  \\
			\end{aligned}
			\right.
			\Leftrightarrow
			\left\{
			\begin{aligned}
			&\bm{w_{r_1}^{1}} \circ \bm{w_{r_2}^{1}} =	\bm{w_{r_1}^{2}} \circ \bm{w_{r_2}^{2}}  \\
			&\bm{b_{r_1}} \circ \bm{w_{r_2}^{2}} + \bm{b_{r_2}} \circ \bm{w_{r_1}^{1}} = 
			\bm{b_{r_1}} \circ \bm{w_{r_2}^{1}} + \bm{b_{r_2}} \circ \bm{w_{r_1}^{2}} = \bm{0}  \\
			\end{aligned}
			\right.
			$
		\end{flushleft}
		That is, the slopes of the straight lines along the same dimension in $r_1$ and $r_2$ are mutually reciprocal, and the intercepts are symmetrical with respect to $\bm{h} = \bm{t}$.
		
		\item \textbf{Composition} (Composition of linear functions, shown in Fig. \ref{Composition}.)
		\begin{flushleft}
			$
			\left\{
			\begin{aligned}
			&\bm{w_{r_2}^{1}} \circ \bm{h} + \bm{b_{r_2}} = \bm{w_{r_2}^{2}} \circ \bm{e}  \\
			&\bm{w_{r_3}^{1}} \circ \bm{e} + \bm{b_{r_3}} = \bm{w_{r_3}^{2}} \circ \bm{t}  \\
			&\bm{w_{r_1}^{1}} \circ \bm{h} + \bm{b_{r_1}} = \bm{w_{r_1}^{2}} \circ \bm{t}  \\
			\end{aligned}
			\right.
			\Leftrightarrow
			\left\{
			\begin{aligned}
			&\bm{w_{r_1}^{1}} = \bm{w_{r_2}^{1}} \circ \bm{w_{r_3}^{1}}, \quad \bm{w_{r_1}^{2}} = \bm{w_{r_2}^{2}} \circ \bm{w_{r_3}^{2}}  \\
			&\bm{b_{r_1}} = \bm{b_{r_2}} \circ \bm{w_{r_3}^{1}} + \bm{b_{r_3}} \circ \bm{w_{r_2}^{2}}  \\
			\end{aligned}
			\right.
			$
		\end{flushleft}
		$r_2$ is a linear mapping from $h$ to $e$, and $r_3$ is a linear mapping from $e$ to $t$, then a new linear mapping from $h$ to $t$ (ie., $r_1$) can be obtained by combining $r_2$ and $r_3$.
	\end{itemize}
\end{proof}

\begin{theorem}
	LineaRE can model 1-to-1, 1-to-N, N-to-1 and N-to-N relations.
\end{theorem}
\begin{proof}
	\textbf{1-to-1:} Obviously, LineaRE can model 1-to-1 relations. \\
	\textbf{1-to-N:} Suppose that we have two observed triplets ($h_0$, $r$, $t_1$) and ($h_0$, $r$, $t_2$):
	\begin{flushleft}
		$
		\left\{
		\begin{aligned}
		&\bm{w_r^{1}} \circ \bm{h_0} + \bm{b_r} = \bm{w_r^{2}} \circ \bm{t_1}  \\
		&\bm{w_r^{1}} \circ \bm{h_0} + \bm{b_r} = \bm{w_r^{2}} \circ \bm{t_2}  \\
		\end{aligned}
		\right.
		\Rightarrow
		\begin{aligned}
		&\bm{w_{r}^{2}}\circ(\bm{t_1}-\bm{t_2}) = \bm{0}
		\end{aligned}
		$
	\end{flushleft}
	$\circled{1} \bm{w_{r}^{2}} = \bm{0}$. \\
	As shown in Fig. \ref{1-to-N-any}, the straight line is one dimension of relation $r$, parallel to $t$ axis. Thus $h_0$ can appear with $\forall t \in \mathbb{R}$ under relation $r$.
	However, $[\bm{w_r^2}]_i$ may actually be a value approximately equal to $\bm{0}$, resulting in a steep slope of the straight line\footnote{If the slope of a line is positive (negative), the larger (smaller) the slope, the steeper it is. We use the term \textit{gentle} slope as the opposite of the \textit{steep} slope in this study.}, as shown in Fig. \ref{1-to-N}.
	Let $\delta$ be the maximum acceptable error, then $h_0$ can appear with multiple $t$ values with low errors, where $t \in [t_1, t_2]$. The steeper the slope is, the larger the range of $t$ values is. Thus, multiple tail entities appearing with the same head entity can be appropriately far away from each other in such dimensions. \\
	$\circled{2} \bm{t_1}-\bm{t_2} = \bm{0}$. \\
	For the dimensions where $\bm{t_1}-\bm{t_2}=\bm{0}$, tail entities are closer to each other. Slopes are not necessarily to be steep in such dimensions. \\
	\textbf{N-to-1}: Similarly, LineaRE can model N-to-1 relations. \\
	\textbf{N-to-N}: N-to-N relations contain both straight lines with steep slopes and straight lines with gentle slopes.
\end{proof}

\begin{corollary}
	The TransE model is a special case of LineaRE.
\end{corollary}
\begin{proof}
	Let $\bm{w_r^{1}} = \bm{w_r^{2}}$, our LineaRE becomes TransE, ie., TransE defines a relation as straight lines with a constant slope of $\bm{1}$, which is a special case of LineaRE.
\end{proof}

\subsection{Loss Function}
A knowledge graph only contains positive triplets, and the way to construct negative triplets is to randomly replace the head or tail entity of an observed triplet with entities in $E$, which is called negative sampling. Many negative sampling methods have been proposed \cite{Kbgan,NSCaching,IGAN}, among which the self-adversarial negative sampling method \cite{RotatE} dynamically adjusts the weights of negative samples according to their scores as the training goes on. We adopt this negative sampling technique. Specifically, the weights (i.e., probability distribution) of negative triplets for a golden triplet ($h$, $r$, $t$) are as follows:
\begin{equation}
	p(
		h_{j}^{'}, r, t_{j}^{'} | \{(h_i^{'}, r, t_i^{'})\}
	)
	=
	\frac{
		\exp(\alpha f_r(\bm{h_{j}^{'}}, \bm{t_{j}^{'}}))
	}
	{
		\sum_{i} \exp(\alpha f_r(\bm{h_{i}^{'}}, \bm{t_{i}^{'}}))
	}
\end{equation}
where $\alpha$ is the temperature of sampling, $\{(h_i^{'}, r, t_i^{'})\}$ are negative triplets for ($h$, $r$, $t$).

Then, we define the logistic loss function for an observed triplet and its negative samples:
\begin{equation}
	softplus(\bm{x}) = \frac{1}{\beta}log(1+exp(\beta \bm{x}))
\end{equation}
\begin{align}
	L & = softplus(f_r(\bm{h}, \bm{t})-\gamma) \nonumber \\
	  & + \sum_{i=1}^{n} p(h_{i}^{'}, r, t_{i}^{'})softplus(\gamma - f_r(\bm{h_{i}^{'}}, \bm{t_{i}^{'}})) \nonumber \\
	  & + \frac{\lambda}{|E|} \sum_{e \in E} \left \| \bm{e} \right \|_{2}^{2}
\end{align}
where $\gamma$ is a fixed margin, $\beta$ is a parameter that can adjust the margin between positive and negative sample scores; $\lambda$ is the regularization coefficient; $E$ is the set of entities in the knowledge graph. Adam \cite{Adam} is used as the optimizer.

\section{Experiments}
In this section, we conduct extensive experiments to evaluate the proposed LineaRE model.

\subsection{Datasets}
Four widely used benchmark datasets are used in our link prediction experiments: FB15k \cite{TransE}, WN18 \cite{TransE}, FB15k-237 \cite{Toutanova2015Observed}, and WN18RR \cite{ConvE}. The statistical information of these datasets is summarized in Table \ref{Datasets}.
\begin{table*}[t]
	\caption{
		Statistical information of the datasets used in experiments.
	}
	\label{Datasets}
	\begin{center}
	\begin{tabular}{|c|c|c|c|c|c||c|c|c|c||c|c|c|c|}
		\hline
		\multirow{2}{*}{\textbf{Datasets}}
		& \multirow{2}{*}{\textbf{\#E}}
		& \multirow{2}{*}{\textbf{\#R}}
		& \multirow{2}{*}{\textbf{\#Train}}
		& \multirow{2}{*}{\textbf{\#Valid}}
		& \multicolumn{9}{c|}{\textbf{\#Test (\%)}}  \\
		\cline{6-14}
		&&&&
		& \textbf{\textit{\#Total}}
		& \textbf{\textit{\#Sym}}	& \textbf{\textit{\#Inv}}	& \textbf{\textit{\#Com}}
		& \textbf{\textit{\#Others}}
		& \textbf{\textit{\#1-to-1}}  & \textbf{\textit{\#1-to-N}}  & \textbf{\textit{\#N-to-1}}
		& \textbf{\textit{\#N-to-N}} \\
		\hline
		FB15k
		& 14,951 & 1,345 & 483,142 & 50,000 & 59,071
		& 7.34	& \textbf{70.22}	& \textbf{22.37}	& 0.06
		& 1.63	& 9.56				& 15.80				& \textbf{73.02} \\
		\hline
		WN18
		& 40,943 & 18	 & 141,442 & 5,000  & 5,000
		& \textbf{21.74}	& \textbf{72.22}	& 3.0				& 3.04
		& 0.84				& \textbf{36.94}	& \textbf{39.62}	& \textbf{22.60}  \\
		\hline
		FB15k-237
		& 14,541	& 237	& 272,115	& 17,535	& 20,466
		& 0		& 0		& \textbf{90.40}	& 9.60
		& 0.94	& 6.32	& \textbf{22.03}	& \textbf{70.72}  \\
		\hline
		WN18RR
		& 40,943 & 11    & 86,835  & 3,034	& 3,134
		& \textbf{34.65}	& 0.29	& 8.33				& \textbf{56.73}
		& 1.34				& 15.16	& \textbf{47.45}	& \textbf{36.06} \\
		\hline
		\multicolumn{14}{l}{
			\#Sym (\#Inv): test triplets ($h$, $r$, $t$) that can be inferred via a directly linked triplet ($t$, $r$, $h$) (($t$, $r^{'}$, $h$)) in train set.
		} \\
		\multicolumn{14}{l}{
			\#Com: test triplets ($h$, $r$, $t$) that can be inferred via a two-step or three-step path ($h$, $p$, $t$) in train set.
		} \\
	\end{tabular}
	\end{center}
\end{table*}
\begin{itemize}
	\item FB15k is a subset of Freebase \cite{Freebase}. We can see that 70.22\% of test triplets ($h$, $r$, $t$) can be inferred via a directly linked triplet ($t$, $r^{'}$, $h$) in the training set, and 22.37\% of test triplets can be inferred via a two-step or three-step path ($h$, $p$, $t$) in the training set. Thus, the key of link prediction on FB15k is to model inversion and composition patterns.
	\item WN18 is a subset of WordNet \cite{WordNet}. The key of link prediction on WN18 is to model inversion and symmetry patterns.
	\item FB15k-237 is a subset of FB15k, where inverse relations are deleted. The key of link prediction on FB15k-237 is to model composition patterns.
	\item WN18RR is a subset of WN18, where inverse relations are deleted. The key of link prediction on WN18RR is to model symmetry patterns.
\end{itemize}
In all these datasets, triplets involving 1-to-1 relations only account for about 1\%. Thus, the methods that can model complex mapping properties will have certain advantages.

\subsection{Experimental Settings}
\paragraph*{Task Description}
Let $\bigtriangleup$ be the test set and $E$ be the set of all entities in the dataset. For each test triplet ($h$, $r$, $t$) $\in \bigtriangleup$, we replace the tail entity $t$ by each entity $e_i \in E$ in turn, forming candidate triplets \{($h$, $r$, $e_i$)\}. Some candidate triplets may exist in the datasets (training, validation, or test sets), and it is common practice to delete them (except the current test triplets). The knowledge graph embedding model is then used to calculate the plausibility of these corrupted triplets and sort them in ascending order. Eventually, the ranks of the correct entities are stored. The prediction process for the head entity is the same.

\paragraph*{Evaluation Protocol}
We report several standard evaluation metrics: the Mean of those predicted Ranks (MR), the Mean of those predicted Reciprocal Ranks (MRR), and the Hits@N (i.e., the proportion of correct entities ranked in the top N, where $\text{N}=1, 3, 10$). A lower MR is better while higher MRR and Hits@N are better.

\paragraph*{Baselines}
We compare the proposed model LineaRE with seven state-of-the-art models\footnote{
	We did not include TransR in the comparison due to its high complexity. Even a NVIDIA GeForce RTX 2080 Ti GPU with 11 GB memory can not run TransR when $k=d=100$, batchsize = 512, and 128 negative samples for each observed triplet. Moreover, reference \cite{TransD} has demonstrated that TransD is better than TransR.
} listed in Table \ref{ScoringFunction} on link prediction tasks. These models used in comparison were published in top-tier AI-related conferences, representing the highest technical level of knowledge graph embedding.

\paragraph*{The fairness of the Comparison}
For the \textbf{fairness} in comparison, all the models except ConvE use \textbf{the same negative sampling technique} (i.e., the self-adversarial negative sampling method proposed in \cite{RotatE}), and the hyperparameters of different models are selected from the same ranges. Because ConvE is quite different from the other models in principle, to report its highest performance, we directly extract the experimental results from the original paper \cite{ConvE}.

\paragraph*{Hyperparameter Settings}
The hyperparameters are selected according to the performance on the validation dataset via grid searching. We set the ranges of hyperparameters as follows: temperature of sampling $\alpha \in$ \{0.5, 1.0\}, fixed margin $\gamma \in$ \{6, 9, 12, 15, 18, 24, 30\}, $\beta$ in softplus $\in$ \{0.75, 1.0, 1.25\}, embedding size $k \in$ \{250, 500, 1000\}, batchsize $b \in$ \{512, 1024, 2048\}, and number of negative samples for each observed triplet $n \in$ \{128, 256, 512, 1024\}. Optimal configurations for our LineaRE are: $\alpha$=1.0, $\beta$=1.25, $\gamma$=15, $k$=1000, $b$=2048 and $n$=128 on FB15k; $\alpha$=0.5, $\beta$=1.25, $\gamma$=6, $k$=500, $b$=1024 and $n$=512 on WN18; $\alpha$=0.5, $\beta$=1.0, $\gamma$=12, $k$=1000, $b$=2048 and $n$=128 on FB15k-237; $\alpha$=0.5, $\beta$=1.0, $\gamma$=12, $k$=1000, $b$=2048 and $n$=128 on WN18RR.

\subsection{Main Results}
The main results on FB15k and WN18 are summarized in Table \ref{FB15kWN18}.
\begin{table*}[t]
	\caption{
		Link prediction results on FB15k and WN18.
	}
	\label{FB15kWN18}
	\begin{center}
		\begin{tabular}{|c||c|c|c|c|c||c|c|c|c|c|}
			\hline
			\multirow{2}{*}{}
			& \multicolumn{5}{c||}{\textbf{FB15k}}
			& \multicolumn{5}{c|}{\textbf{WN18}}   \\
			\cline{2-6} \cline{7-11}
			& \textbf{\textit{MR}}  & \textbf{\textit{MRR}}	 & \textbf{\textit{Hits@1}}  & \textbf{\textit{Hits@3}}  & \textbf{\textit{Hits@10}}
			& \textbf{\textit{MR}}  & \textbf{\textit{MRR}}  & \textbf{\textit{Hits@1}}  & \textbf{\textit{Hits@3}}  & \textbf{\textit{Hits@10}} \\
			\hline
			TransE \cite{TransE}
			& 34			& .737	& .650	& .799	& .874
			& \textbf{145}	& .821	& .713	& .930	& .955  \\
			TransH \cite{TransH}
			& \textbf{32}	& .748	& .661	& .813	& .884
			& 452			& .823	& .721	& .929	& .954  \\
			TransD \cite{TransD}
			& 33	& .750	& .664	& .817	& .886
			& 261	& .822	& .722	& .926	& .956  \\
			DistMult \cite{DistMult}
			& 59	& .789	& .730	& .830	& .887
			& 496	& .810	& .694	& .922	& .949  \\
			ComplEx \cite{ComplEx}
			& 63	& .809	& .757	& .846	& .894
			& 531	& .948	& .945	& .949	& .953  \\
			ConvE \cite{ConvE}
			& 64	& .745	& .670	& .801	& .873
			& 504	& .942	& .935	& .947	& .955  \\
			RotatE \cite{RotatE}
			& 40	& .797	& .746	& .830	& .884
			& 309	& .949	& .944	& .952	& .959  \\
			\hline
			LineaRE
			& 36	& \textbf{.843}	& \textbf{.805}	& \textbf{.867}	& \textbf{.906}
			& 170	& \textbf{.952}	& \textbf{.947}	& \textbf{.955}	& \textbf{.961}  \\
			\hline
		\end{tabular}
	\end{center}
\end{table*}
LineaRE significantly outperforms all those previous state-of-the-art models on almost all the metrics except that TransX performs slightly better than LineaRE on the metric MR on FB15k. Table \ref{FB15k-237WN18RR} summarizes the results on FB15k-237 and WN18RR.
\begin{table*}[t]
	\caption{
		Link prediction results on FB15k-237 and WN18RR.
	}
	\label{FB15k-237WN18RR}
	\begin{center}
		\begin{tabular}{|c||c|c|c|c|c||c|c|c|c|c|}
			\hline
			\multirow{2}{*}{}
			& \multicolumn{5}{c||}{\textbf{FB15k-237}}
			& \multicolumn{5}{c|}{\textbf{WN18RR}}   \\
			\cline{2-6} \cline{7-11}
			& \textbf{\textit{MR}}  & \textbf{\textit{MRR}}	 & \textbf{\textit{Hits@1}}  & \textbf{\textit{Hits@3}}  & \textbf{\textit{Hits@10}}
			& \textbf{\textit{MR}}  & \textbf{\textit{MRR}}  & \textbf{\textit{Hits@1}}  & \textbf{\textit{Hits@3}}  & \textbf{\textit{Hits@10}} \\
			\hline
			TransE \cite{TransE}
			& 172	& .334	& .238	& .371	& .523
			& 1730	& .242	& .042	& .406	& .541  \\
			TransH \cite{TransH}
			& 168	& .339	& .243	& .375	& .531
			& 4345	& .233	& .044	& .395	& .524  \\
			TransD \cite{TransD}
			& 172	& .330	& .235	& .365	& .518
			& 2701	& .247	& .062	& .401	& .537  \\
			DistMult \cite{DistMult}
			& 301	& .311	& .225	& .341	& .485
			& 4675	& .439	& .407	& .450	& .502  \\
			ComplEx \cite{ComplEx}
			& 376	& .313	& .227	& .342	& .486
			& 4824	& .466	& .438	& .479	& .526  \\
			ConvE \cite{ConvE}
			& 246	& .316	& .239	& .350	& .491
			& 5277	& .46	& .39	& .43	& .48  \\
			RotatE \cite{RotatE}
			& 177	& .338	& .241	& .375	& .533
			& 3340	& .476	& .428	& .492	& .571  \\
			\hline
			LineaRE
			& \textbf{155}	& \textbf{.357}	& \textbf{.264}	& \textbf{.391}	& \textbf{.545}
			& \textbf{1644}	& \textbf{.495}	& \textbf{.453}	& \textbf{.509}	& \textbf{.578}  \\
			\hline
		\end{tabular}
	\end{center}
\end{table*}
We can find that no previous model performs better than our LineaRE on any metric.

In general, if a model can cover the main connectivity patterns and mapping properties in a dataset, it will perform well on this dataset. For example, ComplEx and our LineaRE cover inverse pattern and complex mapping properties, they both perform well on FB15k and WN18.
\textit{But there is an exception}, DistMult achieves good performance on FB15k although it cannot model the antisymmetry and inversion patterns. The reason given by \cite{RotatE} is that for most of the relations in FB15K, the types of head entities and tail entities are different. For ($h$, $r$, $t$), the triplet ($t$, $r$, $h$) is usually impossible to be valid since the entity type of $t$ does not match the head entity type of $h$.

\textit{Another exception} is that RotatE performs better than ComplEx on WN18, in which the composition pattern is negligible, while the complex mapping properties are important. The reason is that triplets involving the complex relation that $hpt_r$ (or $tph_r$)
\footnote{
	For relation $r$, $tph_r$ denotes the average number of tails per head, and $hpt_r$ denotes the average number of head per tail.
} 
is greater than $10$ account for $0.89\%$ and $65.1\%$ in WN18 and FB15k, respectively. Thus, the advantage of ComplEx against RotatE in dealing with complex relations is relatively small on WN18. We can find that ComplEx achieves better performance than RotatE on FB15k thanks to the capability of dealing with complex mapping properties.

\paragraph*{Results of TransX on WN18RR}
TransX appears extremely poor Hit@1 on WN18RR. The reason is that TransX makes the translation vector be $\bm{0}$ for symmetry relations, and $f_{r}(\bm{h},\bm{h}) = \|\bm{h_{\bot}} + \bm{0} - \bm{h_{\bot}}\|_{1/2} = 0$, then the entity $h$ will rank at $1$ when predicting ($h$, $r$, $?$), i.e., Hit@1 $=0$ for almost all symmetry relations. If we remove the candidate entity $h$ when predicting ($h$, $r$, $?$) and remove the candidate entity $t$ when predicting ($?$, $r$, $t$), then TransE, TransH and TransD have Hit@1 $=0.313, 0.316, 0.324$, respectively.

\subsection{Ablation Analysis}
To evaluate the importance of connectivity patterns and complex mapping properties, in this section, we analyze the performance of these models with respect to the connectivity patterns and mapping properties in detail (Refer to Table \ref{ModelingAbility}, which summarizes the modeling capabilities of these models).

\paragraph{Symmetry (RotatE and TransE)}
Among these methods, the difference between RotatE and TransE is only that the former can model symmetric relations and the latter cannot. The performance of RotatE is significantly better than that of TransE, because there are many symmetric relations in all datasets except FB15k-237, especially in WN18RR.

\paragraph{Antisymmetry and Inversion (ComplEx and DistMult)}
Complex embeddings enable ComplEx to model two more connectivity patterns (antisymmetry and inversion) than DistMult. The former performs better than the latter on all datasets, especially on WN18, in which the main relation patterns are antisymmetry and inversion.

\paragraph{Composition (LineaRE and ComplEx)}
Complex can model not only all the connectivity patterns except composition but also the complex mapping properties, which makes it achieve very good performance on all datasets other than on FB15k-237, in which the main connectivity pattern is composition. DistMult, which cannot model composition patterns, also performs poorly on FB15k-237. The difference between our LineaRE and ComplEx is that LineaRE is capable of modeling the composition pattern. Thus, our model performs better, especially on FB15k-237.

\paragraph{Complex mapping properties (LineaRE and RotatE)}
RotatE has a powerful modeling capability for all the above connectivity patterns, which makes it perform well on these datasets. However, RotatE is still inferior to our LineaRE because our LineaRE has the same capability of modeling all the connectivity patterns as RotatE does, and further, LineaRE can deal with complex mapping properties that RotatE cannot handle. On the relatively more complex dataset FB15k, our LineaRE archives a more prominent advantage. Recall that triplets involving a complex relation, $hpt_r$ (or $tph_r$) of which is greater than $10$, account for $0.89\%$ and $65.1\%$ in datasets WN18 and FB15k, respectively.

\subsection{Experimental Results on FB15k by Relation Category}
Following \cite{TransE,TransH,TransD,DistMult,RotatE}, we also did some further investigation on the performance of LineaRE on different relation categories
\footnote{
	Following \cite{TransH}, for each relation $r$, we compute $tph_r$ and $hpt_r$. If $hpt_r < 1.5$ and $tph_r < 1.5$, $r$ is treated as one-to-one; if $hpt_r \geq 1.5$ and $tph_r \geq 1.5$, $r$ is treated as a many-to-many; if $hpt_r < 1.5$ and $tph_r \geq 1.5$, $r$ is treated as one-to-many. If $hpt_r \geq 1.5$ and $tph_r < 1.5$, $r$ is treated as many-to-one.
}.
Table \ref{RelationCategery} summarizes the detailed results by relation category
on FB15k, which shows that our LineaRE achieves the best performance on complex relations.
\begin{table*}[t]
	\caption{
		The detailed link prediction results by relation category on FB15k.
	}
	\label{RelationCategery}
	\centering
	\begin{tabular}{|c||c|c|c|c||c|c|c|c|}
		\hline
		\textbf{Rel. Cat}
		& \textbf{1-to-1}	& \textbf{1-to-N}	& \textbf{N-to-1}	& \textbf{N-to-N}	& \textbf{1-to-1}	& \textbf{1-to-N}	& \textbf{N-to-1}	& \textbf{N-to-N}  \\
		\hline
		\textbf{Task}
		& \multicolumn{4}{c||}{\textbf{Predicting Head (Hits@10)}}
		& \multicolumn{4}{c||}{\textbf{Predicting Tail (Hits@10)}}   \\
		\hline
		TransE \cite{TransE}	& .916	& \textbf{.975}	& .626	& .881	& .899	& .704	& \textbf{.968}	& .909  \\
		TransH \cite{TransH}	& .892	& .969	& .628	& .895	& .866	& .716	& .965	& .921  \\
		TransD \cite{TransH}	& .903	& .971	& .639	& .895	& .880	& .741	& .964	& .921  \\
		DistMult \cite{DistMult}& .925	& .965	& .657	& .890	& .923	& .821	& .949	& .917  \\
		ComplEx	\cite{ComplEx} & .928	& .962	& .673	& .897	& \textbf{.934}	& .831	& .950	& .923  \\
		RotatE \cite{RotatE}	& .922	& .967	& .602	& .893	& .923	& .713	& .961	& .922  \\
		\hline
		LineaRE			& \textbf{.930}	& .973	& \textbf{.723}	& \textbf{.906}	& .923	& \textbf{.854}	& .965	& \textbf{.933}  \\
		\hline
		\hline
		\textbf{Task}
		& \multicolumn{4}{c||}{\textbf{Predicting Head (MRR)}}
		& \multicolumn{4}{c||}{\textbf{Predicting Tail (MRR)}}   \\
		\hline
		TransE \cite{TransE}	& .740	& .929	& .498	& .730	& .739	& .594	& .906	& .752  \\
		TransH \cite{TransH}	& .664	& .905	& .495	& .750	& .658	& .586	& .891	& .774  \\
		TransD \cite{TransH}	& .669	& .910	& .499	& .752	& .664	& .601	& .894	& .775  \\
		DistMult \cite{DistMult}& .813	& .922	& .526	& .793	& .805	& .683	& .886	& .817  \\
		ComplEx \cite{ComplEx}	& .820	& .928	& .557	& .819	& .815	& .717	& .890	& .838  \\
		RotatE \cite{RotatE}	& \textbf{.878}	& .934	& .465	& .787	& \textbf{.872}	& .611	& .909	& .832  \\
		\hline
		LineaRE	& .865	& \textbf{.943}	& \textbf{.623}	& \textbf{.845}	& .860	& \textbf{.772}	& \textbf{.920}	& \textbf{.867}  \\
		\hline
	\end{tabular}
\end{table*}
\begin{figure*}[t]
	\centering
	\subfigure[\_similar\_to]{
		\includegraphics[width=0.23\linewidth]{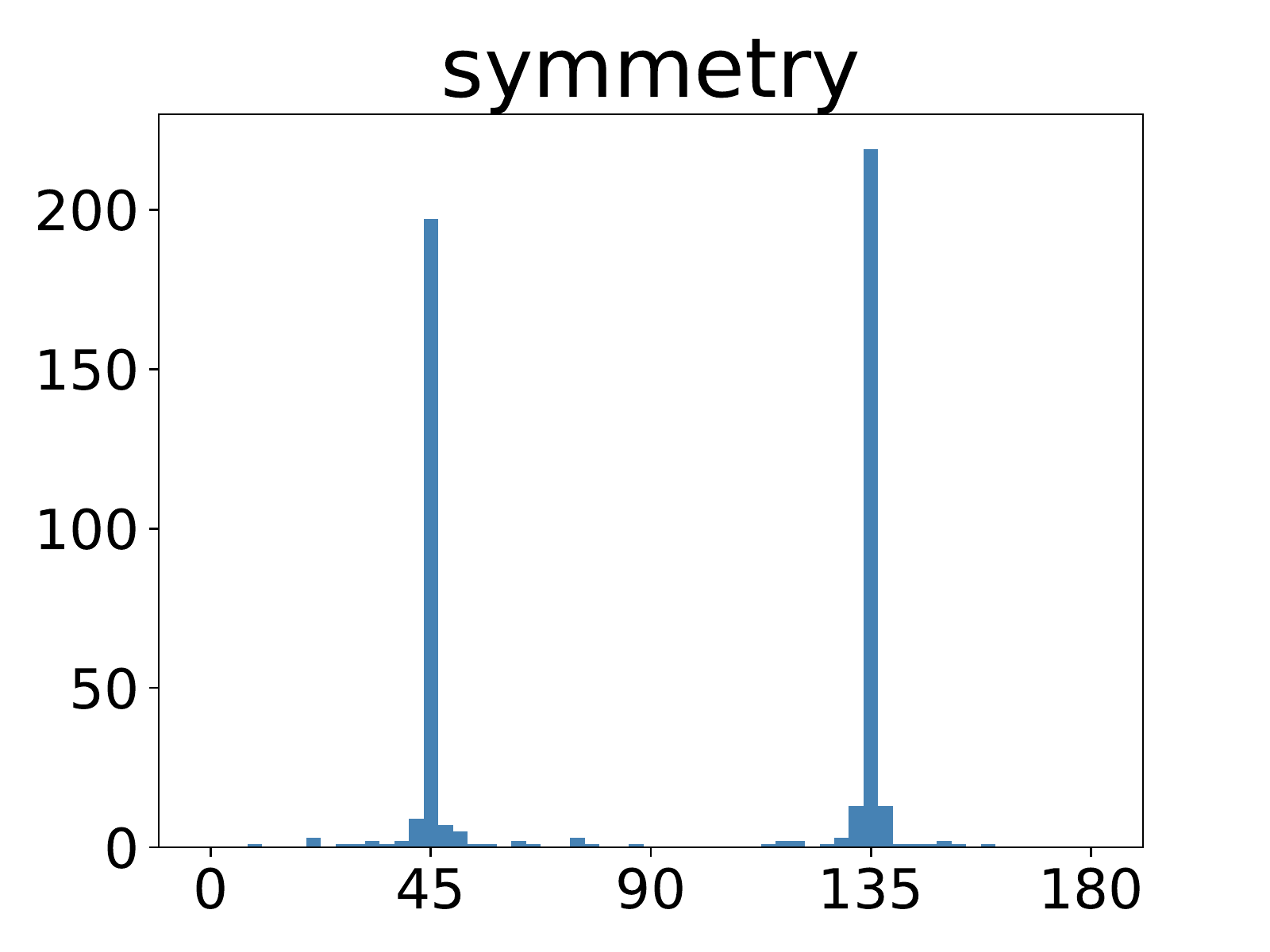}
		\label{SymmetryEmbedding}
	}
	\subfigure[\_hypernym \& \_hyponym$^{-1}$]{
		\includegraphics[width=0.23\linewidth]{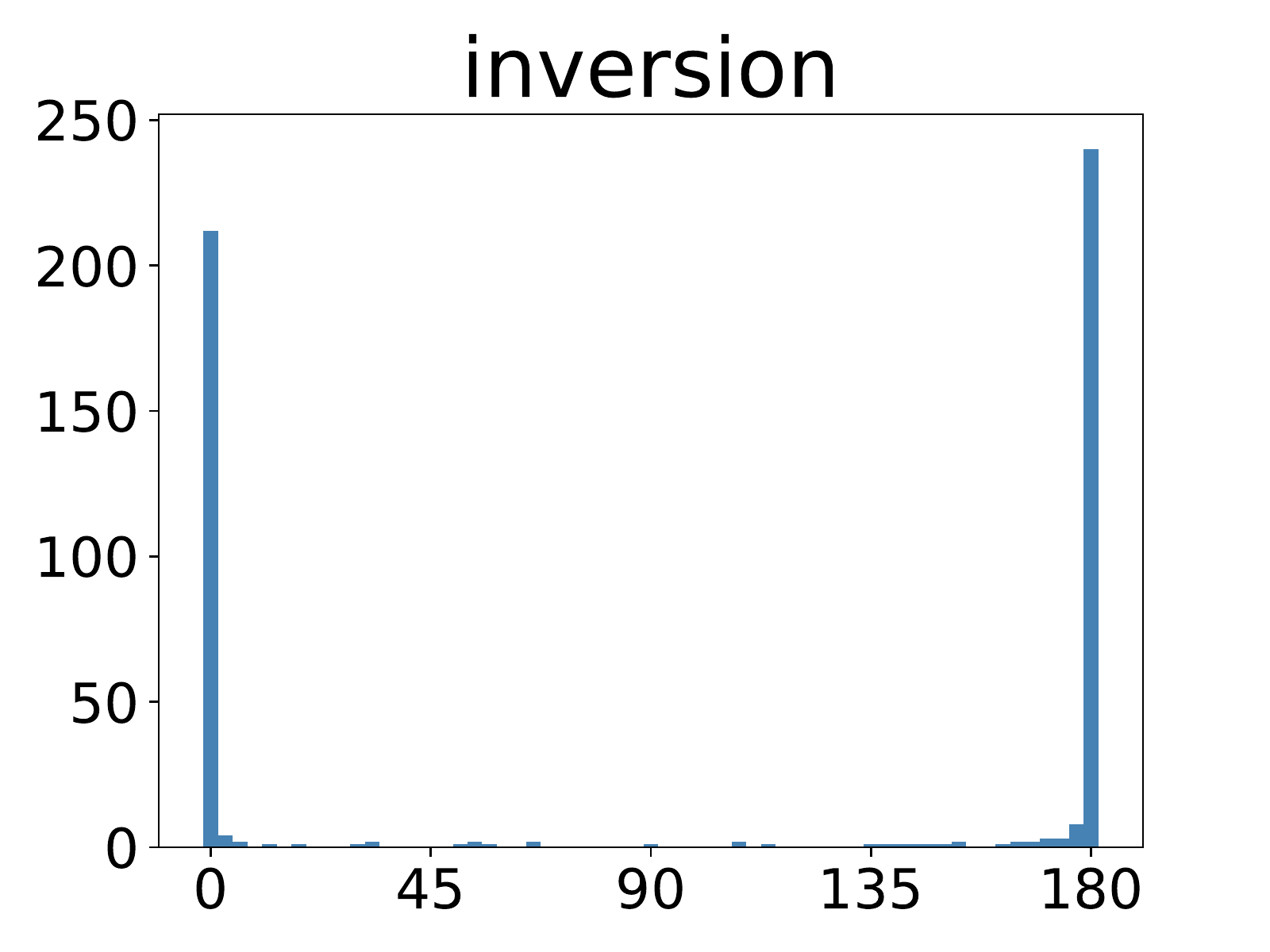}
		\label{InversionEmbedding}
	}
	\subfigure[winner$\odot$for$_1$ \& for$_2$]{
		\includegraphics[width=0.23\linewidth]{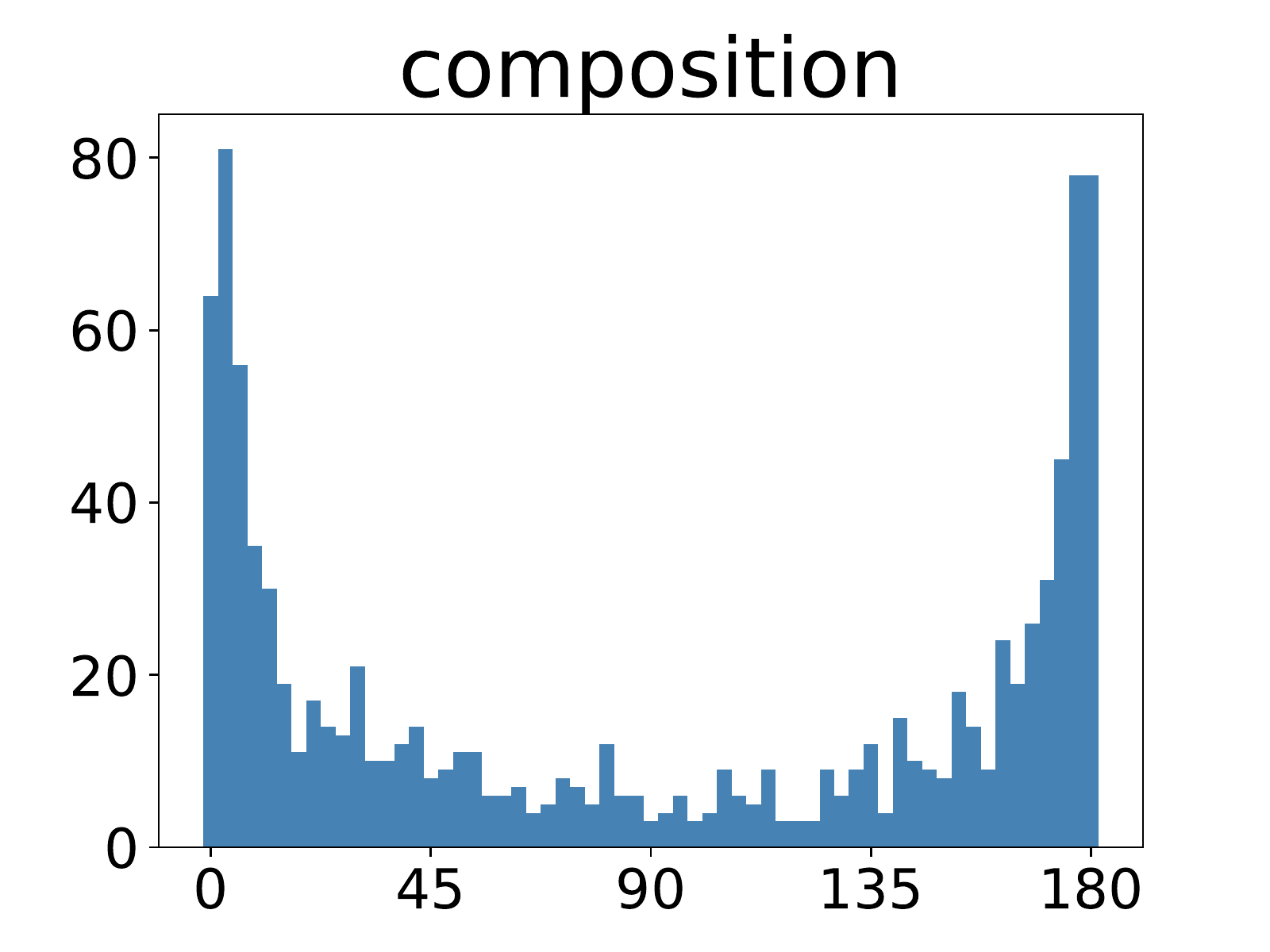}
		\label{CompositionEmbedding}
	}
	\subfigure[\_hyponym]{
		\includegraphics[width=0.23\linewidth]{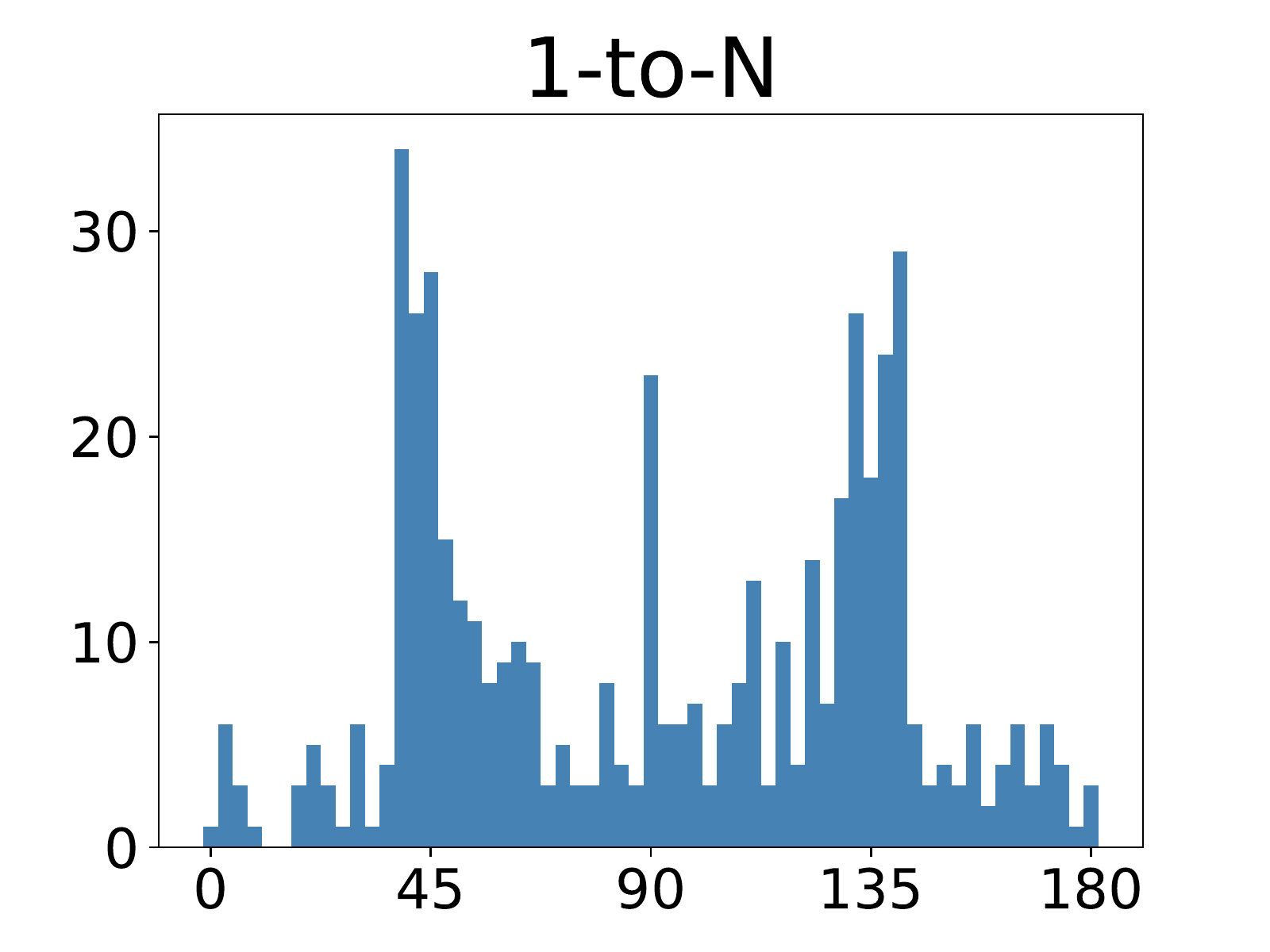}
		\label{1-to-NEmbedding}
	}
	\caption{
		Histograms of angles corresponding to some relation embeddings. (a) Angles between each straight line of $\_similar\_to$ and the $h$ axis; (b) Angles between the straight lines of  $\_hypernym$ and those of $\_hyponym^{-1}$ along the same dimension; (c) Angles between the straight lines of $for_2$ and the composition of $for_1$ and $winner$; (d) Angles between the straight lines of $\_hyponym$ and the $h$ axis; $\odot$ denotes the composition operation.
	}
	\label{RelationEmbedding}
\end{figure*}
DistMult, ComplEx and LineaRE, which are capable of modeling complex mapping properties, perform well on 1-to-N (\textit{predicting tail}), N-to-1 (\textit{predicting head}), and N-to-N relations, while RotatE and TransE both perform worse.

The performance of TransH and TransD are worse than expected. We believe that these two models reduce their ability in other aspects such as modeling inversion pattern when enhancing their ability to deal with complex relations.

\subsection{Investigation of Entity and Relation Embeddings}
To verify our theoretical analysis of the modeling capabilities of LineaRE in Section \ref{Capabilities}, we investigate some relevant entity and relation embeddings (500 dimensions on WN18 and 1000 dimensions on FB15k-237).

\paragraph{Symmetry Pattern}
Fig. \ref{SymmetryEmbedding} shows the angles between the straight lines of the relation $\_similar\_to$ in WN18 and the $h$ axis. Almost all of the 500 angles are equal to or close to 45° or 135°, i.e. these straight lines are symmetrical with respect to $h=t$. This provides the evidence for our analysis of LineaRE in modeling the symmetry pattern in Section \ref{Capabilities} (Fig. \ref{Symmetry}).
As shown in Fig. \ref{RelationLines}, we select two representative dimensions of $\_similar\_to$ and plot the corresponding straight lines. In Fig. \ref{SymmetryLine1}, the angle between the first straight line and $h$ axis is 135°, and the angle between the second straight line and $h$ axis is 45°. Moreover, the entity pairs ($h$, $t$) are closely distributed on or near the line in our LineaRE, while in TransE, the entity pairs are scattered. This indicates that our LineaRE has better modeling ability than TransE.
\begin{figure*}[t]
	\centering
	\subfigure[$\_similar\_to$ in LineaRE]{
		\centering
		\includegraphics[width=0.23\linewidth]{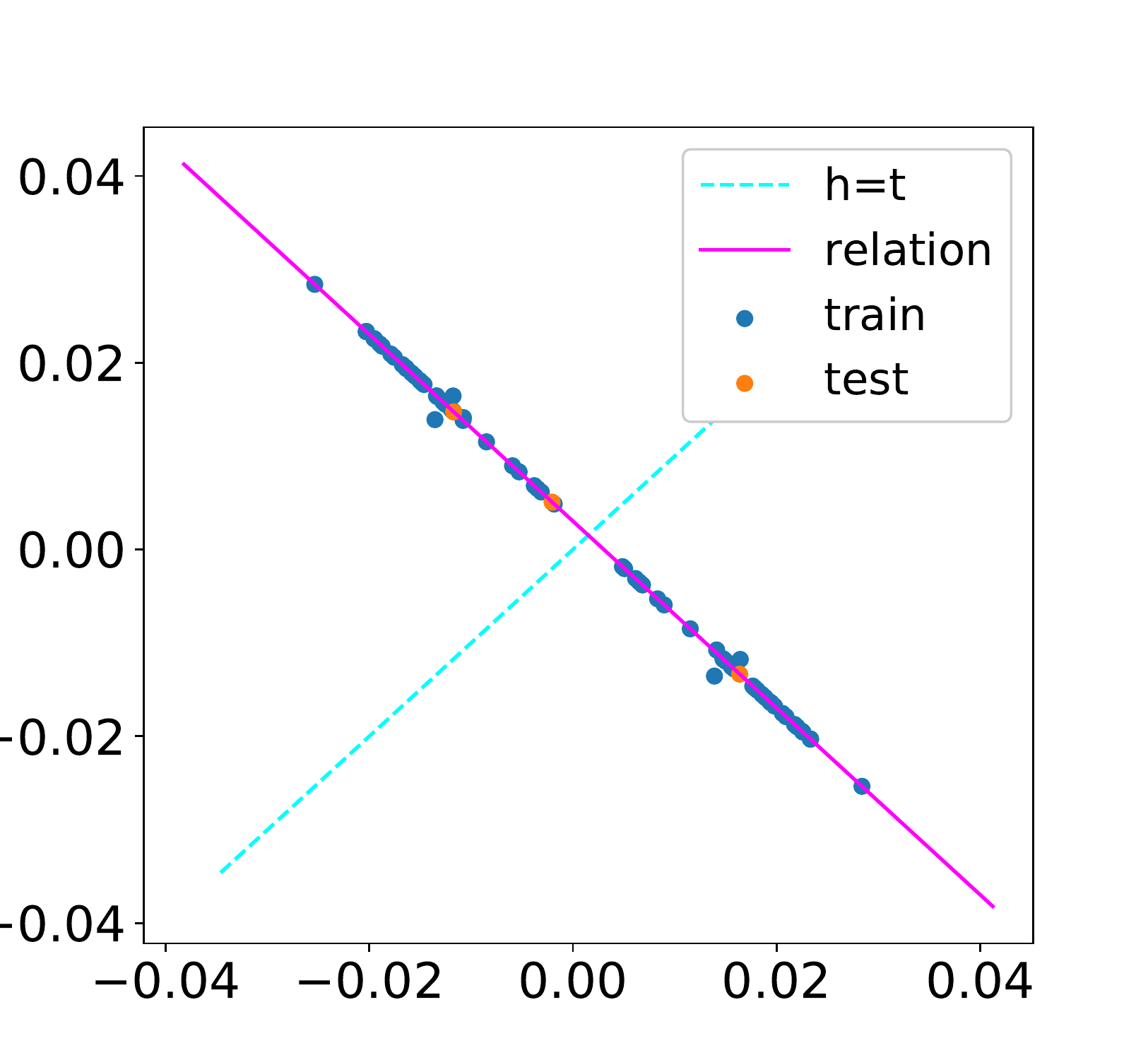}
		\includegraphics[width=0.23\linewidth]{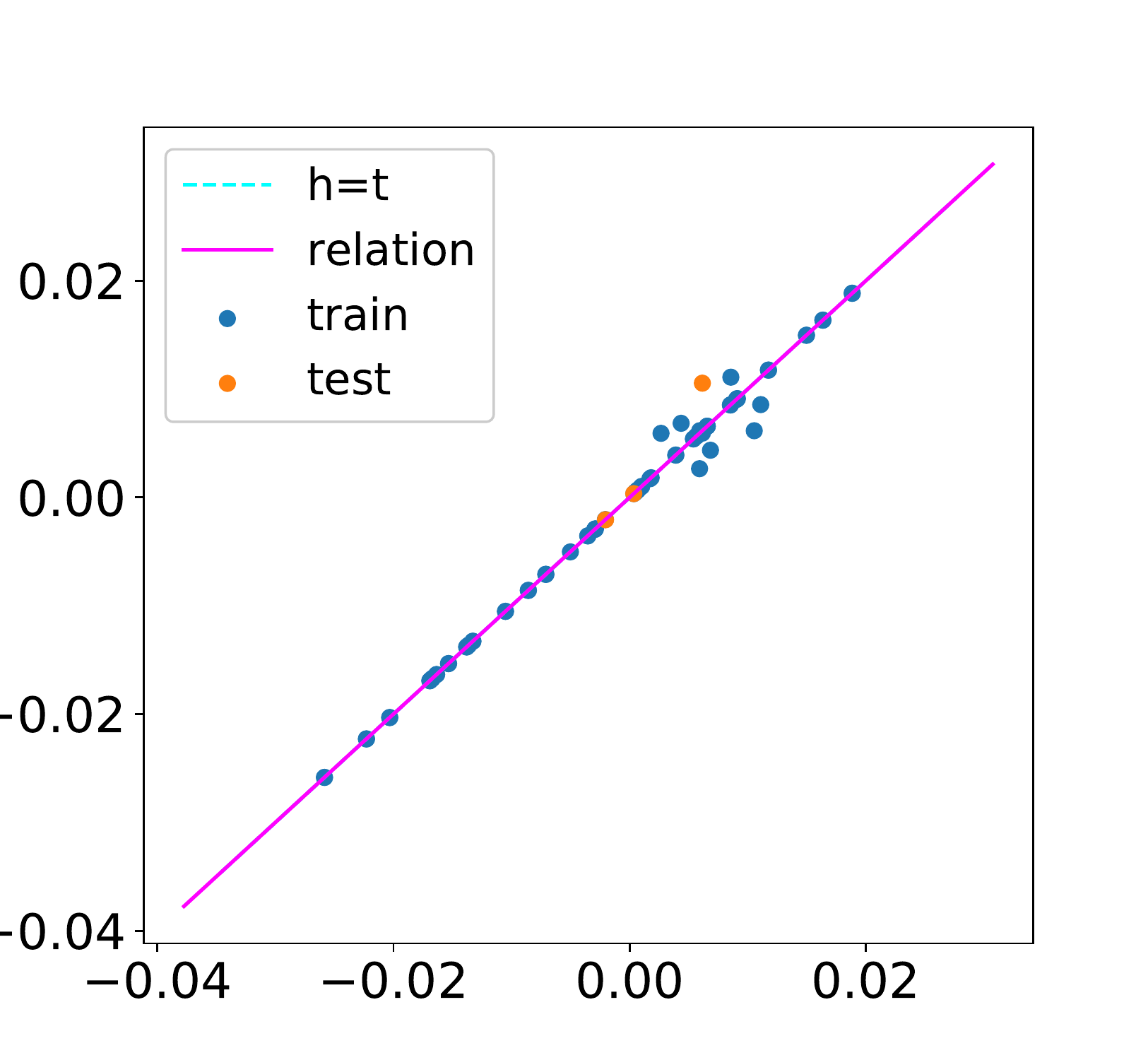}
		\label{SymmetryLine1}
	}
	\subfigure[$\_similar\_to$ in TransE]{
		\includegraphics[width=0.23\linewidth]{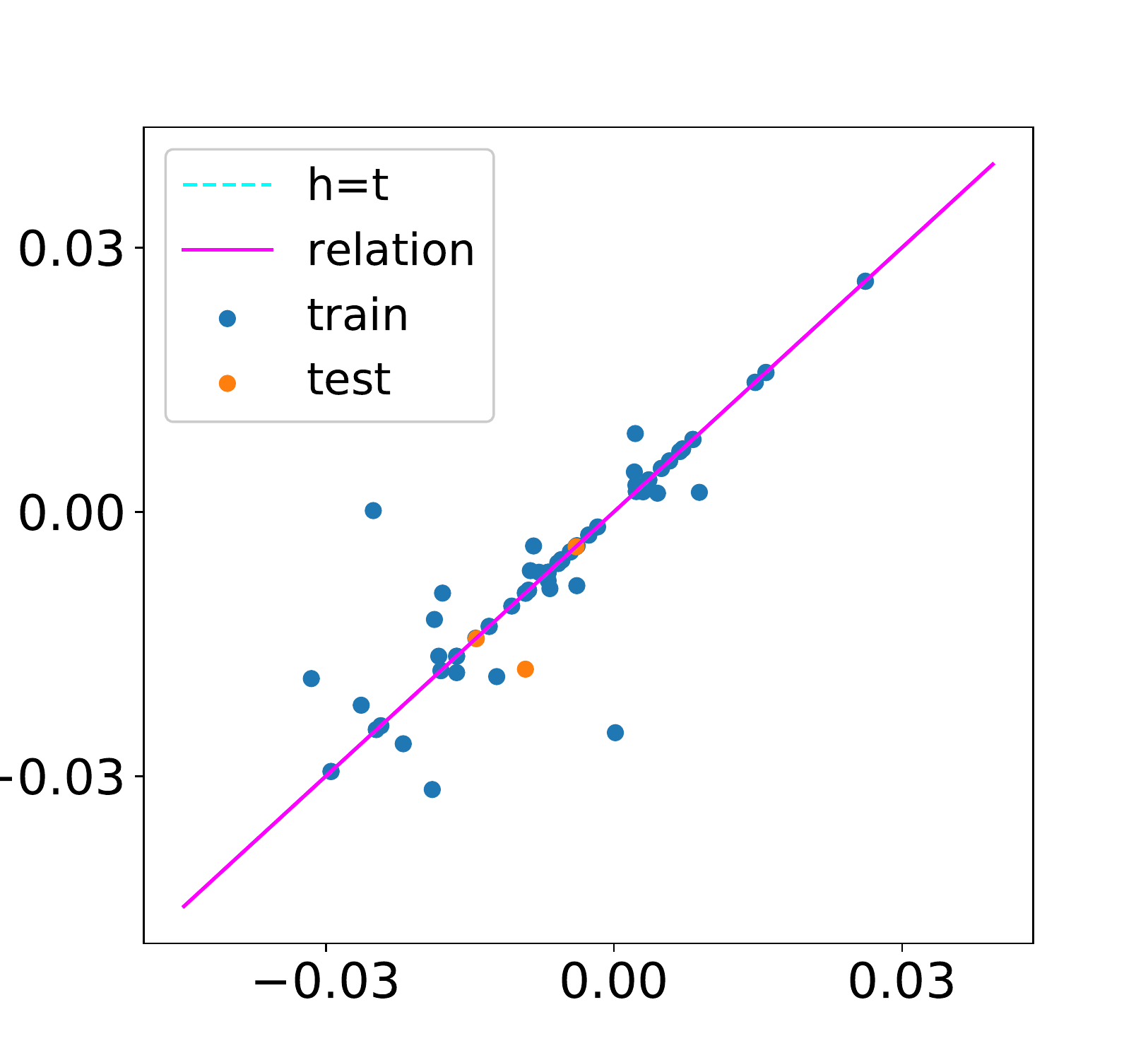}
		\includegraphics[width=0.23\linewidth]{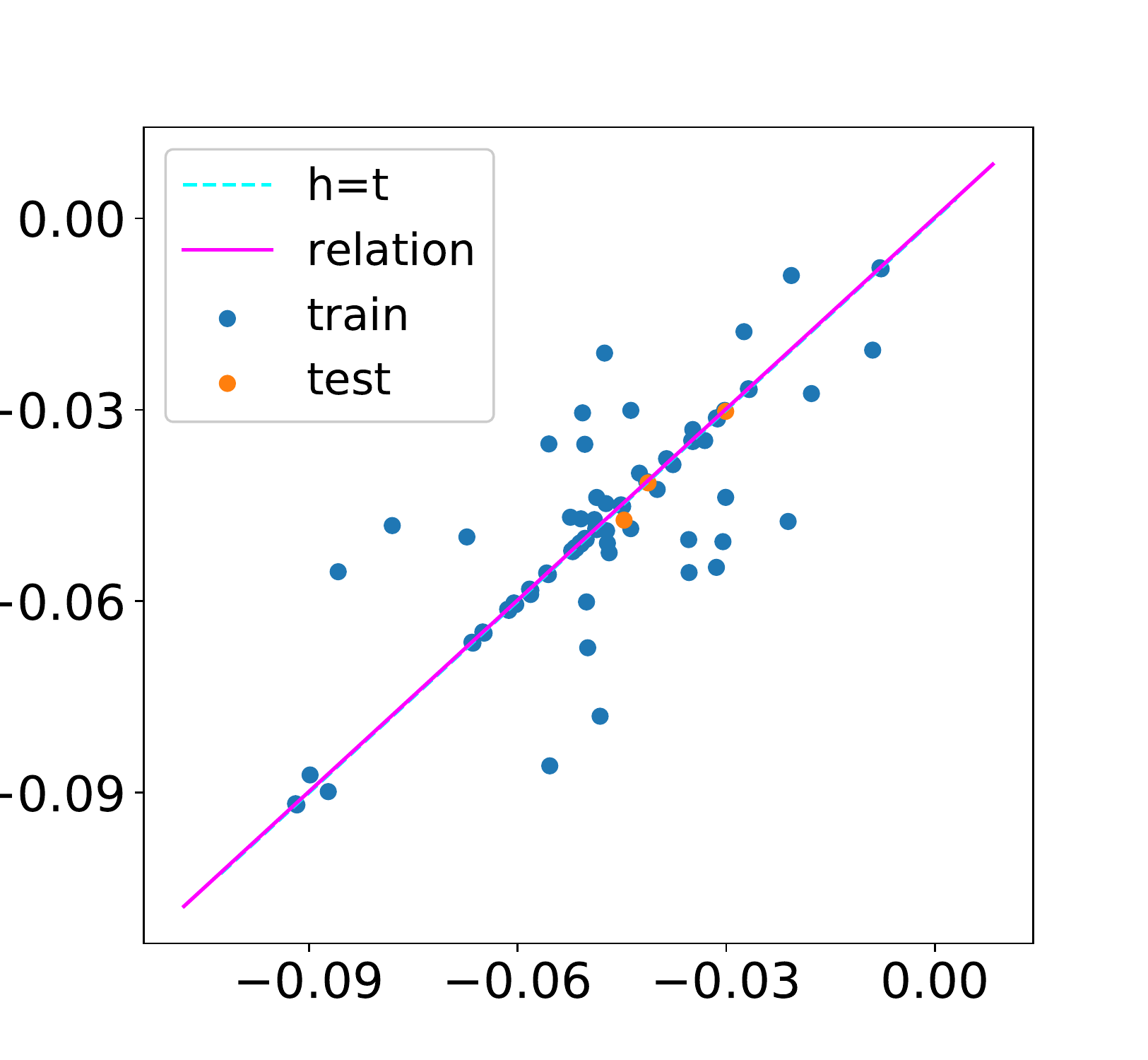}
		\label{SymmetryLine2}
	}
	\caption{
		Visualization of the straight lines of $\_similar\_to$ and the entity pairs having the relation of $\_similar\_to$.
	}
	\label{RelationLines}
\end{figure*}

\paragraph{Antisymmetry and Inversion Pattern}
Relation $\_hypernym$ and $\_hyponym$ in WN18 are a pair of inverse relations. We first inverse $\_hyponym$ (denoted as $\_hyponym^{-1}$), and then calculate the angles between the straight lines of the two relations along the same dimensions. Fig. \ref{InversionEmbedding} shows that most angles are equal to or close to 0° or 180°, i.e., the lines of $\_hypernym$ and $\_hyponym^{-1}$ are parallel to each other. To further prove that $\_hypernym$ and $\_hyponym$ are symmetrical with respect to $h=t$, we select two dimensions of them and plot the corresponding straight lines, as shown in Fig. \ref{hypernym}. In addition, we can see from Fig. \ref{has_part} that $\_has\_part$ and $\_part\_of$ in WN18 are also symmetrical with respect to $h=t$. These figures are completely consistent with Fig. \ref{Inversion} in Section \ref{Capabilities}, which indicates that LineaRE can model the inverse pattern.
\begin{figure}[t]
	\centering
	\subfigure[$\_hypernym \& \_hyponym$]{
		\includegraphics[width=0.46\linewidth]{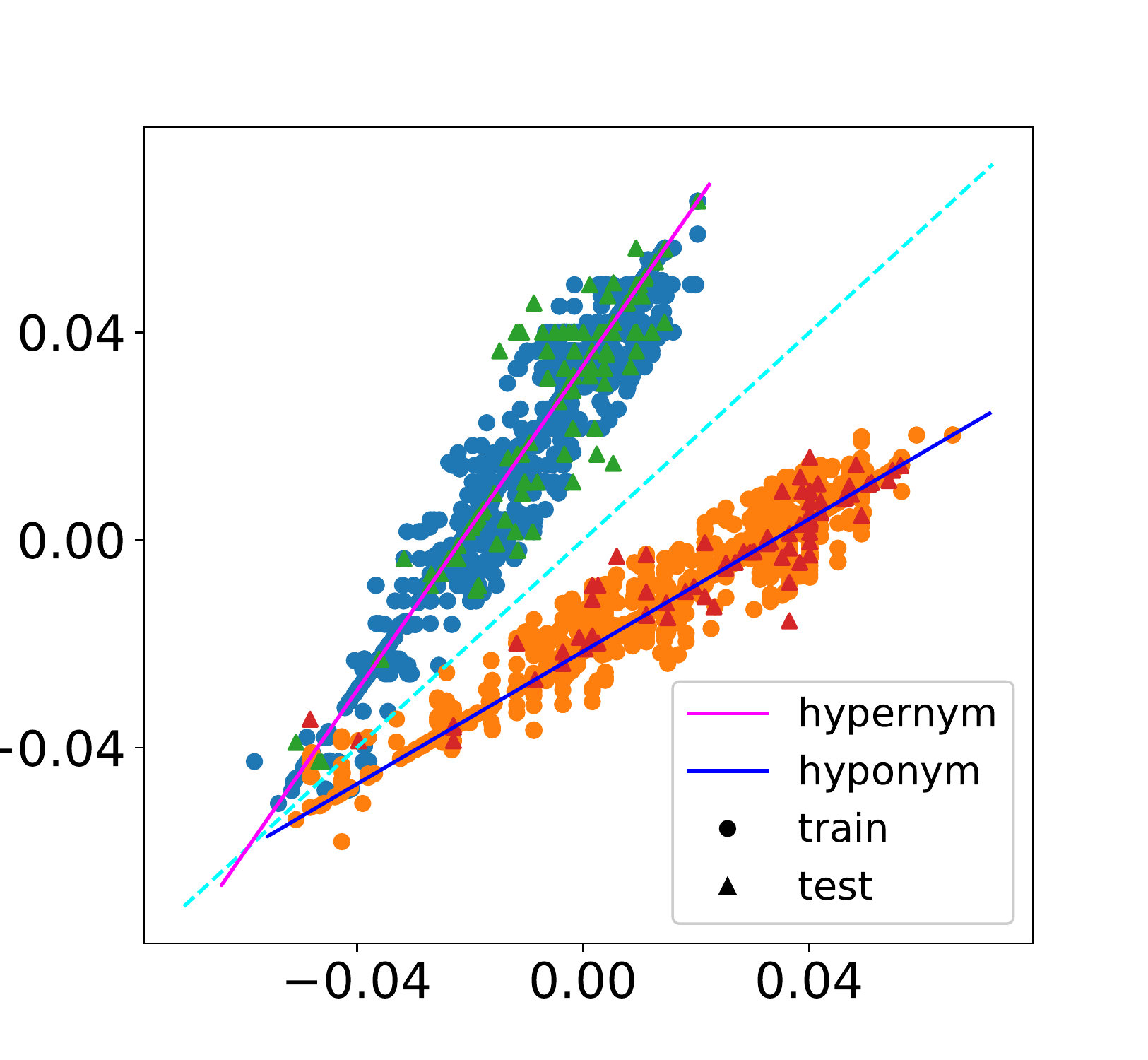}
		\includegraphics[width=0.46\linewidth]{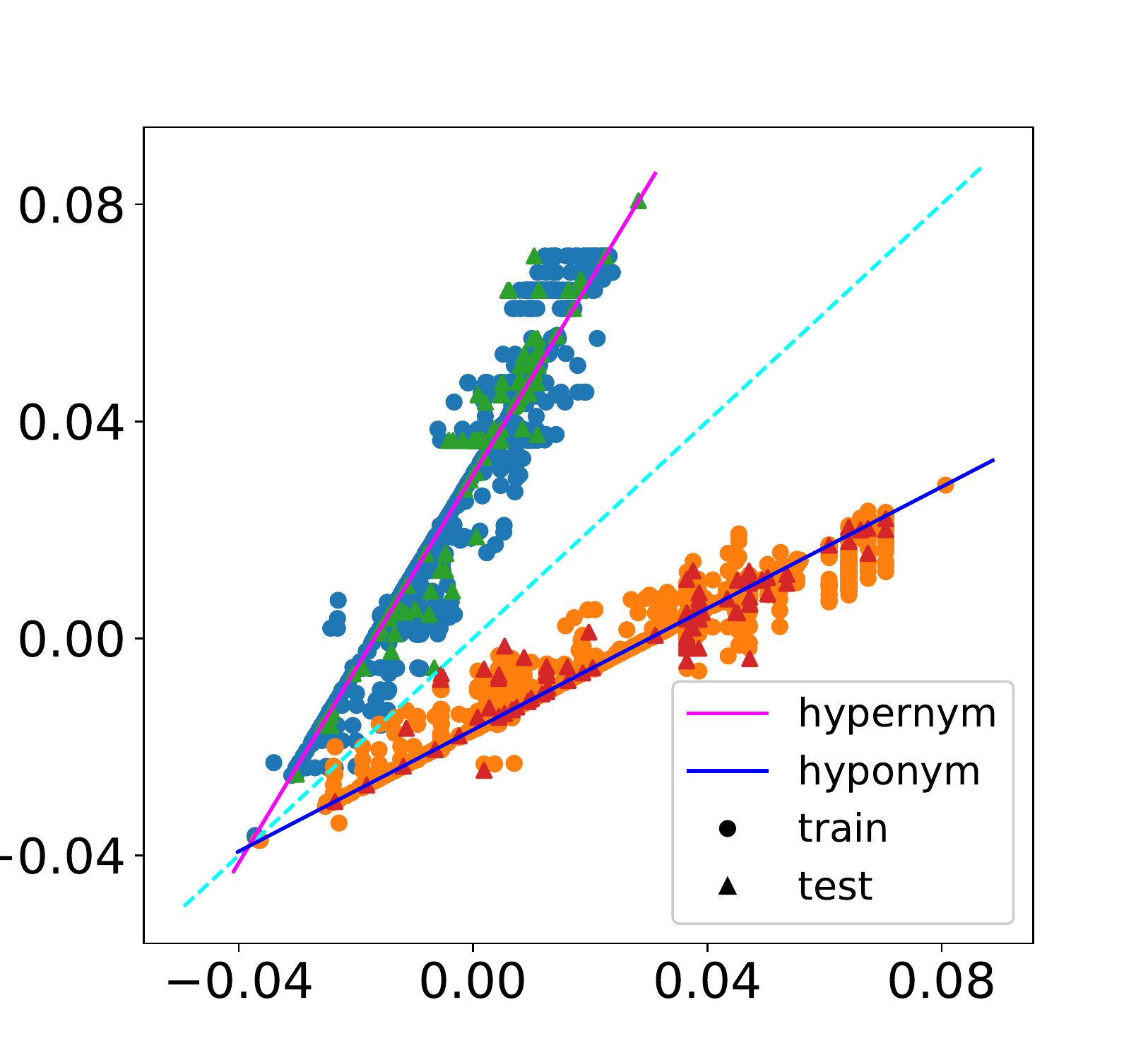}
		\label{hypernym}
	}
	\subfigure[$\_has\_part \& \_part\_of$]{
		\includegraphics[width=0.46\linewidth]{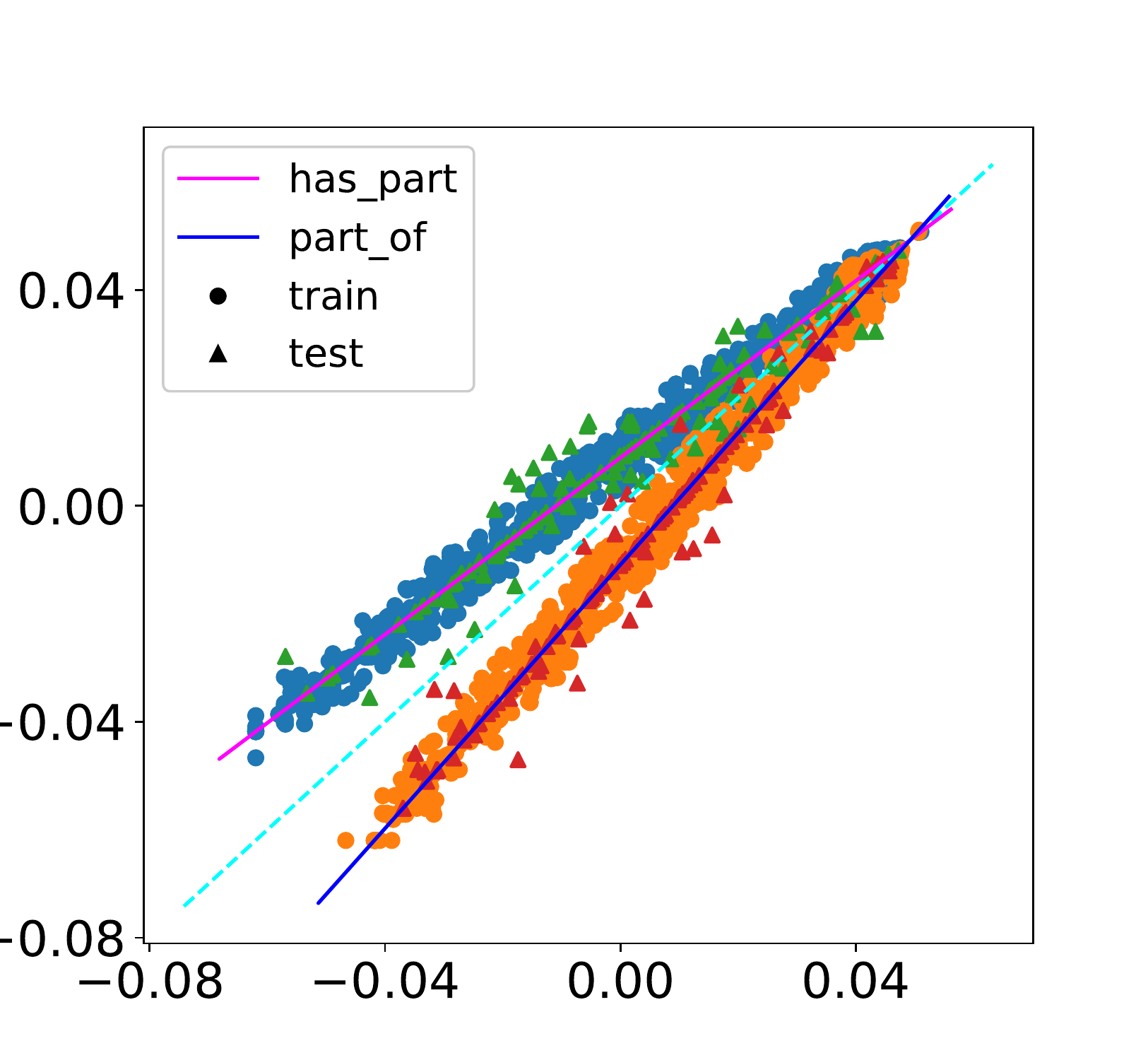}
		\includegraphics[width=0.46\linewidth]{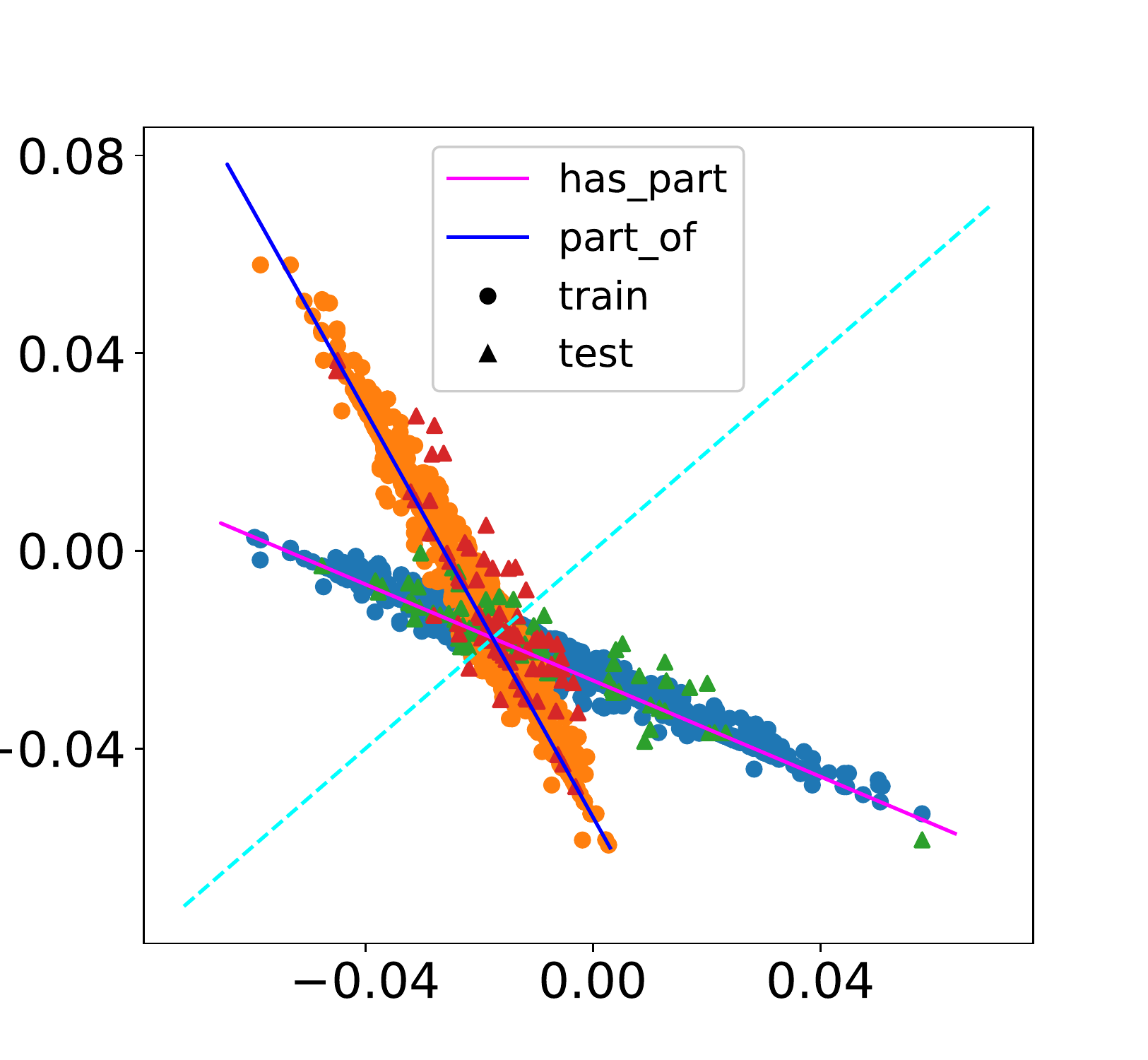}
		\label{has_part}
	}
	\caption{
		Visualization of the straight lines of inverse relations.
	}
	\label{InversionLine}
\end{figure}

\paragraph{Composition Pattern}
In FB15k-237, $for_2$ is a composition of $for_1$ and $winner$\footnote{
	Here, $for_1$ represents relation /award/award\_nominee/award\_nominations./awa-
	rd/award\_nomination/nominated\_for, $winner$ represents relation/award/ award\_category/winners./award/award\_honor/award\_winner, and $for_2$ represents /award/award\_category/nominees./award/award\_nomination/nominated\_ for.
}. We compute the angles between the composite straight lines and the lines of $for_2$ along the same dimensions. Fig. \ref{CompositionEmbedding} shows that the composition of $for_1$ and $winner$ is very similar to $for_2$.

\paragraph{Complex Mapping Properties}
For the 1-to-N relation $\_hyponym$, $hpt_r=1.02$, and $tph_r=3.66$, Fig. \ref{1-to-NEmbedding} shows that there are more steep-slope straight lines than the gentle-slope ones in $\_hyponym$. The relation $\_member\_of\_domain\_topic$ (denoted as $topic$) in WN18 is a 1-to-N relation, and ($06090869$, $topic$) has $69$ tails in the training set, and $3$ tails in the test set. Fig. \ref{TopicLine} shows two straight lines of $topic$ and the distribution of these ($06090869$, $t$) points. The first straight line is very steep, which allows the 72 tail entities corresponding to the head entity $06090869$ to be different from each other. And the second straight line is relatively flat, which makes the 72 tail entities close to each other. Thus, these tail entities are different from each other in dimensions like Fig. \ref{topic_1}, and the same in dimensions like Fig. \ref{topic_2}. This is completely consistent with our analysis in Section \ref{Capabilities}, which proves that LineaRE can model complex mapping properties.
\begin{figure}[t]
	\centering
	\subfigure[steep slope]{
		\includegraphics[width=0.46\linewidth]{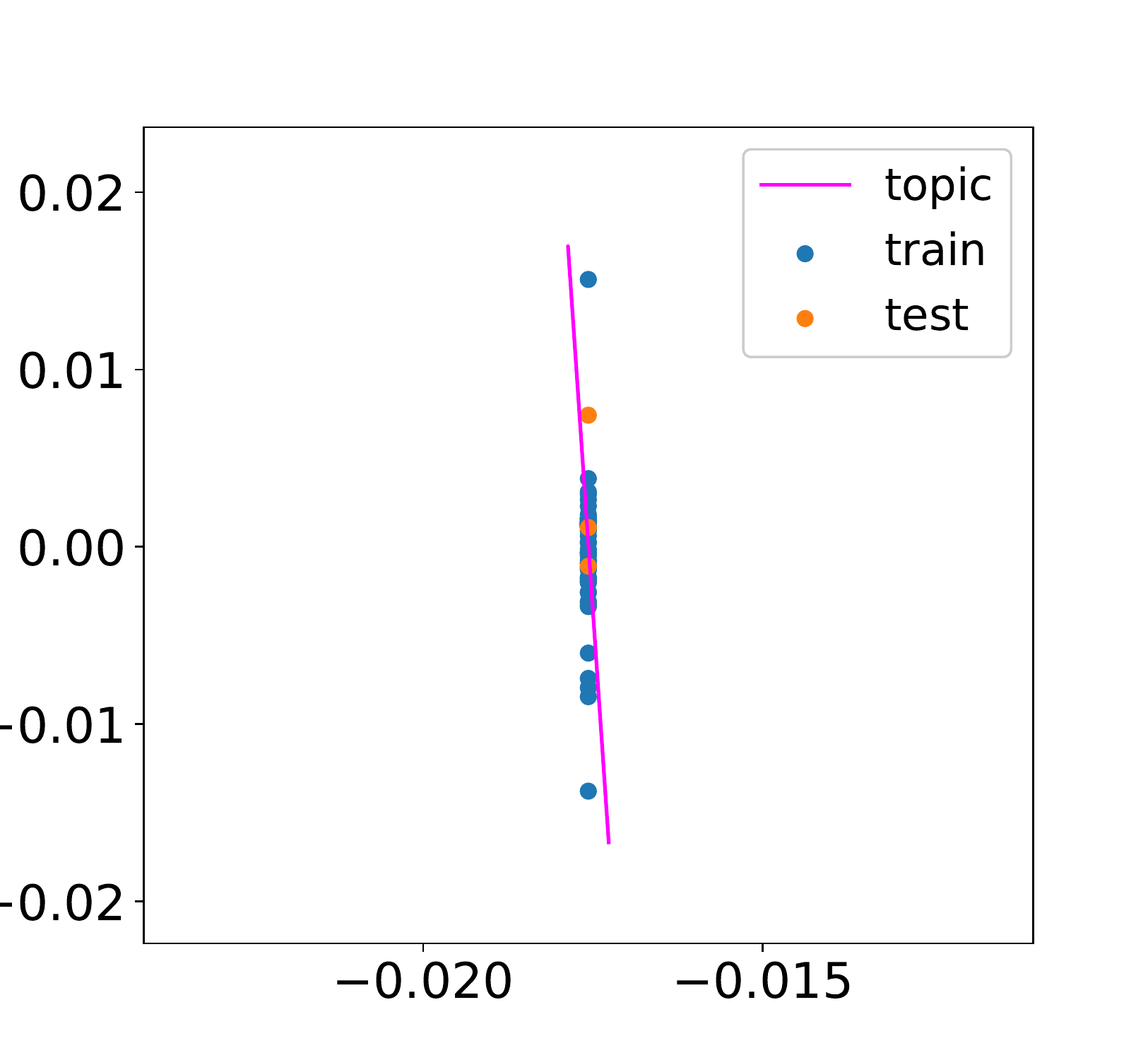}
		\label{topic_1}
	}
	\subfigure[gentle slope]{
		\includegraphics[width=0.46\linewidth]{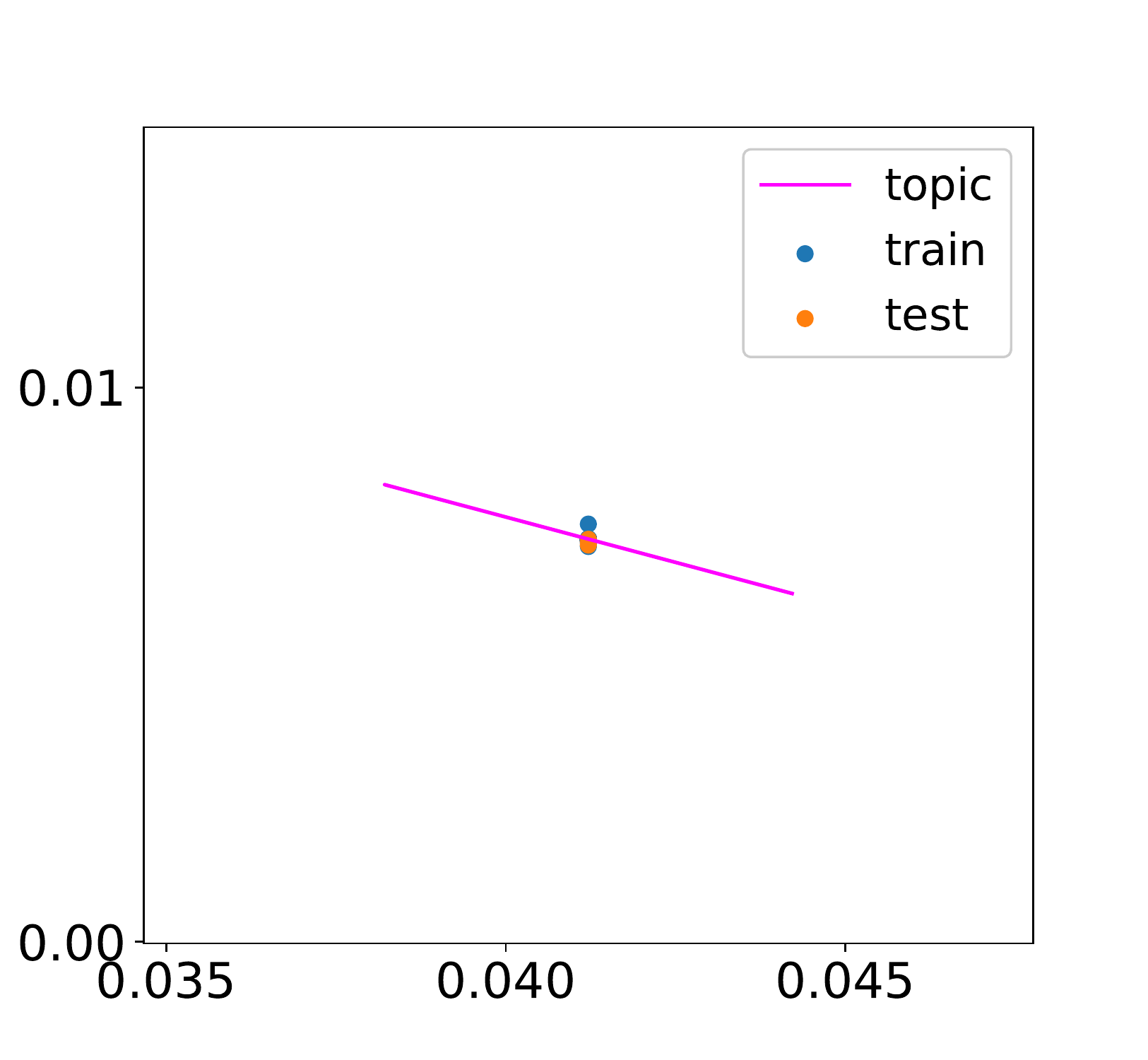}
		\label{topic_2}
	}
	\caption{
		Visualization of the straight lines of relation $topic$ and 72 entity pairs ($06090869$, $t_i$).
	}
	\label{TopicLine}
\end{figure}

\subsection{Inferring Relation Patterns on Dataset Countries}
Following \cite{Bouchard2015On,nickel2016holographic,RotatE}, we also evaluate our proposed LineaRE on the Countries dataset \cite{Bouchard2015On, nickel2016holographic}, which is meticulously designed for testing the capabilities of knowledge graph embedding methods when modeling composition pattern. It contains 2 relations ($neighborOf$, $locatedIn$) and 272 entities (244 countries, 5 regions, and 23 subregions). Reference \cite{nickel2016holographic} sets three tasks with increasing difficulty:
\begin{enumerate}
	\item Task S1 requires inferring a simple composition pattern:
	\begin{center}
		$locatedIn(c,s) \wedge locatedIn(s,r) \Rightarrow locatedIn(c,r)$
	\end{center}
	where $c$ denotes a country, $s$ denotes a subregion, $r$ denotes a region.
	\item Task S2 is more difficult than task S1, the correct triplets can be predicted from:
	\begin{center}
		$neighborOf(c_1,c_2) \wedge locatedIn(c_2,r) \Rightarrow locatedIn(c_1,r)$
	\end{center}
	\item Task S3 is the most difficult task, the correct triplets can be predicted from:
	\begin{center}
		$neighborOf(c_1,c_2) \wedge locatedIn(c_2,s) \wedge locatedIn(s,r) \Rightarrow locatedIn(c_1,r)$
	\end{center}
\end{enumerate}
In Table \ref{Countries}, we report the results with respect to the AUC-PR metric (area under the precision-recall curve). Due to the ability of modeling both the composition pattern and complex mapping properties, LineaRE achieves the best performance, especially on the most difficult task S3.

\begin{table}[t]
	\caption{Results on the Countries dataset}
	\begin{center}
		\begin{tabular}{|c|c|c|c|}
			\hline
			\textbf{}&\multicolumn{3}{|c|}{\textbf{Countries(AUC-PR)}} \\
			\cline{2-4} 
			\textbf{} & \textbf{\textit{S1}}& \textbf{\textit{S2}}& \textbf{\textit{S3}} \\
			\hline
			TransE
			& \textbf{1.00 $\pm$ 0.00}	& \textbf{1.00 $\pm$ 0.00}	& 0.96 $\pm$ 0.00  \\
			ComplEx
			& \textbf{1.00 $\pm$ 0.00}	& 0.98 $\pm$ 0.00	& 0.88 $\pm$ 0.01  \\
			RotatE
			& \textbf{1.00 $\pm$ 0.00}	& \textbf{1.00 $\pm$ 0.00}	& 0.95 $\pm$ 0.00  \\
			\hline
			LineaRE
			& \textbf{1.00 $\pm$ 0.00}	& \textbf{1.00 $\pm$ 0.00}	& \textbf{0.99 $\pm$ 0.01}  \\
			\hline
			\multicolumn{4}{l}{Results of TransE, ComplEx and RotatE are taken from \cite{RotatE}.}
		\end{tabular}
		\label{Countries}
	\end{center}
\end{table}

\section{Conclusion}
In this paper, we proposed a novel knowledge graph embedding method LineaRE for link prediction, which models four connectivity patterns and four mapping properties of relations  in the manner of linear regression. We provided formal mathematical proofs to demonstrate the modeling capabilities of LineaRE. Extensive experimental results on the task of link prediction showed that the proposed LineaRE model significantly outperforms existing state-of-the-art models on several widely used benchmark datasets. A deep investigation into the entity and relation embeddings further verifies our theoretical analysis of its modeling capabilities. Moreover, results on the Countries dataset showed that our LineaRE has a powerful reasoning ability.

\section*{Acknowledgment}
This work was supported by the National Natural Science Foundation of China under grants 62076130 and 91846104, and the National Key Research and Development Program of China under grant 2018AAA0102002.

\bibliographystyle{IEEETran}
\bibliography{conference_101719}

\end{document}